\documentclass{article}
\usepackage{amssymb}
\usepackage{amsfonts}
\usepackage{amsmath}

\newtheorem{proposition}{Proposition}
\newtheorem{theorem}[proposition]{Theorem}

\newtheorem{note}[proposition]{Note}

\newtheorem{notation}[proposition]{Notation}

\newtheorem{convention}[proposition]{Convention}

\newtheorem{corollary}[proposition]{Corollary}

\newtheorem{definition}[proposition]{Definition}

\newtheorem{lemma}[proposition]{Lemma}

\newtheorem{remark}[proposition]{Remark}

\newenvironment{proof}[1][Proof]{\textbf{#1.} }{\  \rule{0.5em}{0.5em}}
\numberwithin{equation}{section}

\input{tcilatex}

\begin{document}

\title{Algebraic function based Banach space valued ordinary and fractional
neural network approximations}
\author{George A. Anastassiou \\
Department of Mathematical Sciences\\
University of Memphis\\
Memphis, TN 38152, U.S.A.\\
ganastss@memphis.edu}
\date{}
\maketitle

\begin{abstract}
Here we research the univariate quantitative approximation, ordinary and
fractional, of Banach space valued continuous functions on a compact
interval or all the real line by quasi-interpolation Banach space valued
neural network operators. These approximations are derived by establishing
Jackson type inequalities involving the modulus of continuity of the engaged
function or its Banach space valued high order derivative of fractional
derivatives. Our operators are defined by using a density function generated
by an algebraic sigmoid function. The approximations are pointwise and of
the uniform norm. The related Banach space valued feed-forward neural
networks are with one hidden layer.
\end{abstract}



\section{Introduction}

The author in \cite{1} and \cite{2}, see Chapters 2-5, was the first to
establish neural network approximation to continuous functions with rated by
very specifically defined neural network operators of Cardaliagnet-Euvrard
and "Squashing" types, by employing the modulus of continuity of the engaged
function or its high order derivative, and producing very tight Jackson type
inequalities. He treats there both the univariate and multivariate cases.
The defining these operators "bell-shaped "and "squashing "functions are
assumed to be of compact suport. Also in \cite{2} he gives the $N$th order
asymptotic expansion for the error of weak approximation of these two
operators to a special natural class of smooth functions, see Chapters 4-5
there.

The author inspired by \cite{14}, continued his studies on neural networks
approximation by introducing and using the proper quasi-interpolation
operators of sigmoidal and hyperbolic tangent type which resulted into \cite%
{3}, \cite{4}, \cite{5}, \cite{6}, \cite{7}, by treating both the univariate
and multivariate cases. He did also the corresponding fractional cases \cite%
{8}, \cite{9}, \cite{13}.

The author here performs algebraic sigmoidal based neural network
approximations to continuous functions over compact intervals of the real
line or over the whole $\mathbb{R}$ with valued to an arbitrary Banach space 
$\left( X,\left \Vert \cdot \right \Vert \right) $. Finally he treats
completely the related $X$-valued fractional approximation. All convergences
here are with rates expressed via the modulus of continuity of the involved
function or its $X$-valued high order derivative, or $X$-valued fractional
derivatives and given by very tight Jackson type inequalities. Iterated
fractional approximation is also included.

Our compact intervals are not necessarily symmetric to the origin. Some of
our upper bounds to error quantity are very flexible and general. In
preparation to prove our results we establish important properties of the
basic density function defining our operators which is induced by algebraic
sigmoidal function.

Feed-forward $X$-valued neural networks (FNNs) with one hidden layer, the
only type of networks we deal with in this article, are mathematically
expressed as 
\begin{equation*}
N_{n}\left( x\right) =\sum \limits_{=0}^{n}c_{j}\sigma \left( \left \langle
a_{j}\cdot x\right \rangle +b_{j}\right) ,\text{ \ \ }x\in \mathbb{R}^{s}%
\text{, \ }s\in \mathbb{N},
\end{equation*}%
where for $0\leq j\leq n$, $b_{j}\in \mathbb{R}$ are the thresholds, $%
a_{j}\in \mathbb{R}^{s}$ are the connection weights, $c_{j}\in X$ are the
coefficients, $\left \langle a_{j}\cdot x\right \rangle $ is the inner
product of $a_{j}$ and $x$, and $\sigma $ is the activation function of the
network. About neural networks in general read \cite{15}, \cite{17}, \cite%
{19}. See also \cite{9} for a complete study of real valued approximation by
neural network operators.

\section{Basics}

We consider the generator algebraic function 
\begin{equation}
\varphi \left( x\right) =\frac{x}{\sqrt[2m]{1+x^{2m}}},\text{ \ }m\in 
\mathbb{N}\text{, }x\in \mathbb{R}\text{,}  \tag{1}  \label{1}
\end{equation}%
which is a sigmoidal type of function and is a strictly increasing function.

We see that $\varphi \left( -x\right) =-\varphi \left( x\right) $ with $%
\varphi \left( 0\right) =0$. We get that 
\begin{equation}
\varphi ^{\prime }\left( x\right) =\frac{1}{\left( 1+x^{2m}\right) ^{\frac{%
2m+1}{2m}}}>0\text{, \ }\forall \text{ }x\in \mathbb{R}\text{,}  \tag{2}
\label{2}
\end{equation}%
proving $\varphi $ as strictly increasing over $\mathbb{R},\varphi ^{\prime
}\left( x\right) =\varphi ^{\prime }\left( -x\right) .$ We easily find that $%
\underset{x\rightarrow +\infty }{\lim }\varphi \left( x\right) =1$, $\varphi
\left( +\infty \right) =1$, and $\underset{x\rightarrow -\infty }{\lim }%
\varphi \left( x\right) =-1$, $\varphi \left( -\infty \right) =-1.$

We consider the activation function 
\begin{equation}
\Phi \left( x\right) =\frac{1}{4}\left[ \varphi \left( x+1\right) -\varphi
\left( x-1\right) \right] .  \tag{3}  \label{3}
\end{equation}%
Clearly it is $\Phi \left( x\right) =\Phi \left( -x\right) ,$ $\forall $ $%
x\in \mathbb{R}$, so that $\Phi $ is an even function and symmetric with
respect to the $y$-axis.

Since $x+1>x-1$, we have $\varphi \left( x+1\right) >\varphi \left(
x-1\right) $ and $\Phi \left( x\right) >0$, $\forall $ $x\in \mathbb{R}$.

Also it is 
\begin{equation}
\Phi \left( 0\right) =\frac{1}{2\sqrt[2m]{2}}.  \tag{4}  \label{4}
\end{equation}%
We observe that 
\begin{equation*}
\Phi ^{\prime }\left( x\right) =\frac{1}{4}\left( \varphi ^{\prime }\left(
x+1\right) -\varphi ^{\prime }\left( x-1\right) \right) =
\end{equation*}%
\begin{equation}
\frac{1}{4}\left( \frac{1}{\left( 1+\left( x+1\right) ^{2m}\right) ^{\frac{%
2m+1}{2m}}}-\frac{1}{\left( 1+\left( x-1\right) ^{2m}\right) ^{\frac{2m+1}{2m%
}}}\right) ,\text{ \ }\forall \text{ }x\in \mathbb{R}.  \tag{5}  \label{5}
\end{equation}%
Let now $x>0$, then $x>-x$ and $\left( x+1\right) ^{2}>\left( x-1\right)
^{2}\geq 0$, implying $\left( x+1\right) ^{2m}>\left( x-1\right) ^{2m}\geq 0$%
, $m\in \mathbb{N}$, and $1+\left( x+1\right) ^{2m}>1+\left( x-1\right)
^{2m}>0$. Consequently it holds 
\begin{equation}
\frac{1}{\left( 1+\left( x-1\right) ^{2m}\right) ^{\frac{2m+1}{2m}}}>\frac{1%
}{\left( 1+\left( x+1\right) ^{2m}\right) ^{\frac{2m+1}{2m}}},  \tag{6}
\label{6}
\end{equation}%
proving $\Phi ^{\prime }\left( x\right) <0$ for $x>0$.

That is $\Phi $ is strictly decreasing over $\left( 0,+\infty \right) $.

Clearly, $\Phi $ is strictly increasing over $\left( -\infty ,0\right) $ and 
$\Phi ^{\prime }\left( 0\right) =0$.

Furthermore we obtain that 
\begin{equation}
\underset{x\rightarrow +\infty }{\lim }\Phi \left( x\right) =\frac{1}{4}%
\left[ \varphi \left( +\infty \right) -\varphi \left( +\infty \right) \right]
=0,  \tag{7}  \label{7}
\end{equation}%
and 
\begin{equation}
\underset{x\rightarrow -\infty }{\lim }\Phi \left( x\right) =\frac{1}{4}%
\left[ \varphi \left( -\infty \right) -\varphi \left( -\infty \right) \right]
=0.  \tag{8}  \label{8}
\end{equation}%
That is the $x$-axis is the horizontal asymptote of $\Phi $.

Conclusion, $\Phi $ is a bell symmetric function with maximum 
\begin{equation}
\Phi \left( 0\right) =\frac{1}{2\sqrt[2m]{2}},\text{ \ }m\in \mathbb{N}. 
\tag{9}  \label{9}
\end{equation}%
We need

\begin{theorem}
\label{t1}We have that 
\begin{equation}
\sum \limits_{i=-\infty }^{\infty }\Phi \left( x-i\right) =1\text{, \ }%
\forall \text{ }x\in \mathbb{R}.  \tag{10}  \label{10}
\end{equation}
\end{theorem}

\begin{proof}
We observe that 
\begin{equation*}
\sum \limits_{i=-\infty }^{\infty }\left( \varphi \left( x-i\right) -\varphi
\left( x-1-i\right) \right) =
\end{equation*}%
\begin{equation*}
\sum \limits_{i=0}^{\infty }\left( \varphi \left( x-i\right) -\varphi \left(
x-1-i\right) \right) +\sum \limits_{i=-\infty }^{-1}\left( \varphi \left(
x-i\right) -\varphi \left( x-1-i\right) \right) .
\end{equation*}%
Furthermore ($\lambda \in \mathbb{Z}^{+}$) 
\begin{equation}
\sum \limits_{i=0}^{\infty }\left( \varphi \left( x-i\right) -\varphi \left(
x-1-i\right) \right) =  \tag{11}  \label{11}
\end{equation}%
\begin{equation*}
\underset{\lambda \rightarrow \infty }{\lim }\sum \limits_{i=0}^{\lambda
}\left( \varphi \left( x-i\right) -\varphi \left( x-1-i\right) \right) \text{
\ (telescoping sum)}
\end{equation*}%
\begin{equation*}
=\underset{\lambda \rightarrow \infty }{\lim }\left( \varphi \left( x\right)
-\varphi \left( x-\left( \lambda +1\right) \right) \right) =1+\varphi \left(
x\right) .
\end{equation*}%
Similarly, it holds 
\begin{equation}
\sum \limits_{i=-\infty }^{-1}\left( \varphi \left( x-i\right) -\varphi
\left( x-1-i\right) \right) =\underset{\lambda \rightarrow \infty }{\lim }%
\sum \limits_{i=-\lambda }^{-1}\left( \varphi \left( x-i\right) -\varphi
\left( x-1-i\right) \right)  \tag{12}  \label{12}
\end{equation}%
\begin{equation*}
=\underset{\lambda \rightarrow \infty }{\lim }\left( \varphi \left(
x+\lambda \right) -\varphi \left( x\right) \right) =1-\varphi \left(
x\right) .
\end{equation*}%
Therefore we derive 
\begin{equation}
\sum \limits_{i=-\infty }^{\infty }\left( \varphi \left( x-i\right) -\varphi
\left( x-1-i\right) \right) =2\text{, \ }\forall \text{ }x\in \mathbb{R}, 
\tag{13}  \label{13}
\end{equation}%
and 
\begin{equation}
\sum \limits_{i=-\infty }^{\infty }\left( \varphi \left( x+1-i\right)
-\varphi \left( x-i\right) \right) =2\text{, \ }\forall \text{ }x\in \mathbb{%
R}.  \tag{14}  \label{14}
\end{equation}%
Adding (\ref{13}) and (\ref{14}) we find 
\begin{equation}
\sum \limits_{i=-\infty }^{\infty }\left( \varphi \left( x+1-i\right)
-\varphi \left( x-1-i\right) \right) =4\text{, \ }\forall \text{ }x\in 
\mathbb{R}.  \tag{15}  \label{15}
\end{equation}%
Clearly, then 
\begin{equation*}
\Phi \left( x-i\right) =\frac{1}{4}\left[ \varphi \left( x+1-i\right)
-\varphi \left( x-1-i\right) \right] ,
\end{equation*}%
proving (\ref{10}).
\end{proof}

We make

\begin{remark}
\label{r2}Because $\Phi $ is even it holds 
\begin{equation*}
\sum \limits_{i=-\infty }^{\infty }\Phi \left( i-x\right) =1\text{, \ }%
\forall \text{ }x\in \mathbb{R}.
\end{equation*}%
Hence 
\begin{equation*}
\sum \limits_{i=-\infty }^{\infty }\Phi \left( i+x\right) =1\text{, \ }%
\forall \text{ }x\in \mathbb{R},
\end{equation*}%
and%
\begin{equation}
\sum \limits_{i=-\infty }^{\infty }\Phi \left( x+i\right) =1\text{, \ }%
\forall \text{ }x\in \mathbb{R}.  \tag{16}  \label{16}
\end{equation}
\end{remark}

\begin{theorem}
\label{t3}It holds 
\begin{equation}
\int_{-\infty }^{\infty }\Phi \left( x\right) dx=1.  \tag{17}  \label{17}
\end{equation}
\end{theorem}

\begin{proof}
We observe that 
\begin{equation}
\int_{-\infty }^{\infty }\Phi \left( x\right) dx=\sum \limits_{j=-\infty
}^{\infty }\int_{j}^{j+1}\Phi \left( x\right) dx=\sum \limits_{j=-\infty
}^{\infty }\int_{0}^{1}\Phi \left( x+j\right) dx=  \tag{18}  \label{18}
\end{equation}%
\begin{equation*}
\int_{0}^{1}\left( \sum \limits_{j=-\infty }^{\infty }\Phi \left( x+j\right)
dx\right) =\int_{0}^{1}1dx=1.
\end{equation*}%
So $\Phi \left( x\right) $ is a density function on $\mathbb{R}.$
\end{proof}

We need

\begin{theorem}
\label{t4}Let $0<\alpha <1$, and $n\in \mathbb{N}$ with $n^{1-\alpha }>2$.
It holds 
\begin{equation}
\sum \limits_{\left \{ 
\begin{array}{c}
k=-\infty \\ 
:\left \vert nx-k\right \vert \geq n^{1-\alpha }%
\end{array}%
\right. }^{\infty }\Phi \left( nx-k\right) <\frac{1}{4m\left( n^{1-\alpha
}-2\right) ^{2m}},\text{ \ }m\in \mathbb{N}.  \tag{19}  \label{19}
\end{equation}
\end{theorem}

\begin{proof}
We have that 
\begin{equation*}
\Phi \left( x\right) =\frac{1}{4}\left[ \varphi \left( x+1\right) -\varphi
\left( x-1\right) \right] ,\text{ \ }\forall \text{ }x\in \mathbb{R}.
\end{equation*}%
Let $x\geq 1$. That is $0\leq x-1<x+1$. Applying the mean value theorem we
get 
\begin{equation}
\Phi \left( x\right) =\frac{1}{4}\cdot 2\cdot \varphi ^{\prime }\left( \xi
\right) =\frac{1}{2\left( 1+\xi ^{2m}\right) ^{\frac{2m+1}{2m}}}>0,  \tag{20}
\label{20}
\end{equation}%
where $0\leq x-1<\xi <x+1$.

Then, 
\begin{equation*}
\left( x-1\right) ^{2}<\xi ^{2}<\left( x+1\right) ^{2}
\end{equation*}%
\begin{equation*}
\left( x-1\right) ^{2m}<\xi ^{2m}<\left( x+1\right) ^{2m}
\end{equation*}%
\begin{equation*}
1+\left( x-1\right) ^{2m}<1+\xi ^{2m}<1+\left( x+1\right) ^{2m}
\end{equation*}%
\begin{equation*}
\left( 1+\left( x-1\right) ^{2m}\right) ^{\frac{2m+1}{2m}}<\left( 1+\xi
^{2m}\right) ^{\frac{2m+1}{2m}}<\left( 1+\left( x+1\right) ^{2m}\right) ^{%
\frac{2m+1}{2m}}
\end{equation*}%
\begin{equation}
\frac{1}{2\left( 1+\xi ^{2m}\right) ^{\frac{2m+1}{2m}}}<\frac{1}{2\left(
1+\left( x-1\right) ^{2m}\right) ^{\frac{2m+1}{2m}}}.  \tag{21}  \label{21}
\end{equation}%
Hence 
\begin{equation}
\Phi \left( x\right) <\frac{1}{2\left( 1+\left( x-1\right) ^{2m}\right) ^{%
\frac{2m+1}{2m}}},\text{ \ }\forall \text{ }x\geq 1.  \tag{22}  \label{22}
\end{equation}%
Thus, we have 
\begin{equation*}
\sum \limits_{\left \{ 
\begin{array}{c}
k=-\infty \\ 
:\left \vert nx-k\right \vert \geq n^{1-\alpha }%
\end{array}%
\right. }^{\infty }\Phi \left( nx-k\right) =\sum \limits_{\left \{ 
\begin{array}{c}
k=-\infty \\ 
:\left \vert nx-k\right \vert \geq n^{1-\alpha }%
\end{array}%
\right. }^{\infty }\Phi \left( \left \vert nx-k\right \vert \right) <
\end{equation*}%
\begin{equation}
\frac{1}{2}\sum \limits_{\left \{ 
\begin{array}{c}
k=-\infty \\ 
:\left \vert nx-k\right \vert \geq n^{1-\alpha }%
\end{array}%
\right. }^{\infty }\frac{1}{\left( 1+\left( \left \vert nx-k\right \vert
-1\right) ^{2m}\right) ^{\frac{2m+1}{2m}}}\leq  \tag{23}  \label{23}
\end{equation}%
\begin{equation*}
\frac{1}{2}\int_{\left( n^{1-\alpha }-1\right) }^{\infty }\frac{1}{\left(
1+\left( x-1\right) ^{2m}\right) ^{\frac{2m+1}{2m}}}dx=\frac{1}{2}%
\int_{n^{1-\alpha }-2}^{\infty }\frac{1}{\left( 1+z^{2m}\right) ^{\frac{2m+1%
}{2m}}}dz=:\left( \ast \right) .
\end{equation*}%
We see that 
\begin{equation*}
1+z^{2m}>z^{2m}
\end{equation*}%
\begin{equation*}
\left( 1+z^{2m}\right) ^{\frac{2m+1}{2m}}>z^{2m+1}
\end{equation*}%
\begin{equation}
\frac{1}{z^{2m+1}}>\frac{1}{\left( 1+z^{2m}\right) ^{\frac{2m+1}{2m}}}. 
\tag{24}  \label{24}
\end{equation}%
Therefore it holds 
\begin{equation*}
\left( \ast \right) <\frac{1}{2}\int_{n^{1-\alpha }-2}^{\infty }\frac{1}{%
z^{2m+1}}dz=\frac{1}{2}\int_{n^{1-\alpha }-2}^{\infty }z^{-\left(
2m+1\right) }dz=
\end{equation*}%
\begin{equation}
\frac{1}{2}\left. \left( \frac{z^{-\left( 2m+1\right) +1}}{-\left(
2m+1\right) +1}\right) \right \vert _{n^{1-\alpha }-2}^{\infty }=\frac{1}{2}%
\left. \left( -\frac{z^{-2m}}{2m}\right) \right \vert _{n^{1-\alpha
}-2}^{\infty }=  \tag{25}  \label{25}
\end{equation}%
\begin{equation*}
\left. \frac{z^{-2m}}{4m}\right \vert _{\infty }^{n^{1-\alpha }-2}=\frac{%
\left( n^{1-\alpha }-2\right) ^{-2m}}{4m}-\frac{\left( \infty \right) ^{-2m}%
}{4m}=\frac{\left( n^{1-\alpha }-2\right) ^{-2m}}{4m},
\end{equation*}%
proving (\ref{19}).
\end{proof}

Denote by $\left \lfloor \cdot \right \rfloor $ the integral part of the
number and by $\left \lceil \cdot \right \rceil $ the ceiling of the number.

\begin{theorem}
\label{t5}Let $\left[ a,b\right] \subset \mathbb{R}$ and $n\in \mathbb{N}$
so that $\left \lceil na\right \rceil \leq \left \lfloor nb\right \rfloor $.
It holds 
\begin{equation}
\frac{1}{\sum \limits_{k=\left \lceil na\right \rceil }^{\left \lfloor
nb\right \rfloor }\Phi \left( nx-k\right) }<2\left( \sqrt[2m]{1+4^{m}}%
\right) ,  \tag{26}  \label{26}
\end{equation}%
$\forall $ $x\in \left[ a,b\right] $, $m\in \mathbb{N}.$
\end{theorem}

\begin{proof}
Let $x\in \left[ a,b\right] $. We see that 
\begin{equation*}
1=\sum \limits_{k=-\infty }^{\infty }\Phi \left( nx-k\right) >\sum
\limits_{k=\left \lceil na\right \rceil }^{\left \lfloor nb\right \rfloor
}\Phi \left( nx-k\right) =
\end{equation*}%
\begin{equation}
\sum \limits_{k=\left \lceil na\right \rceil }^{\left \lfloor nb\right
\rfloor }\Phi \left( \left \vert nx-k\right \vert \right) >\Phi \left( \left
\vert nx-k_{0}\right \vert \right) ,  \tag{27}  \label{27}
\end{equation}%
$\forall $ $k_{0}\in \left[ \left \lceil na\right \rceil ,\left \lfloor
nb\right \rfloor \right] \cap \mathbb{Z}$.

We can choose $k_{0}\in \left[ \left \lceil na\right \rceil ,\left \lfloor
nb\right \rfloor \right] \cap \mathbb{Z}$ such that $\left \vert
nx-k_{0}\right \vert <1.$

Therefore we get that 
\begin{equation}
\Phi \left( \left \vert nx-k_{0}\right \vert \right) >\Phi \left( 1\right) =%
\frac{1}{4}\left( \frac{2}{\sqrt[2m]{1+2^{2m}}}\right) =\frac{1}{2\sqrt[2m]{%
1+2^{2m}}},  \tag{28}  \label{28}
\end{equation}%
and 
\begin{equation}
\sum \limits_{k=\left \lceil na\right \rceil }^{\left \lfloor nb\right
\rfloor }\Phi \left( \left \vert nx-k\right \vert \right) >\frac{1}{2\sqrt[2m%
]{1+2^{2m}}}.  \tag{29}  \label{29}
\end{equation}%
That is 
\begin{equation}
\frac{1}{\sum \limits_{k=\left \lceil na\right \rceil }^{\left \lfloor
nb\right \rfloor }\Phi \left( \left \vert nx-k\right \vert \right) }<2\sqrt[%
2m]{1+4^{m}},  \tag{30}  \label{30}
\end{equation}%
proving the claim.
\end{proof}

We make

\begin{remark}
\label{r6}We also notice that 
\begin{equation*}
1-\sum \limits_{k=\left \lceil na\right \rceil }^{\left \lfloor nb\right
\rfloor }\Phi \left( nb-k\right) =\sum \limits_{k=-\infty }^{\left \lceil
na\right \rceil -1}\Phi \left( nb-k\right) +\sum \limits_{k=\left \lfloor
nb\right \rfloor +1}^{\infty }\Phi \left( nb-k\right)
\end{equation*}%
\begin{equation}
>\Phi \left( nb-\left \lfloor nb\right \rfloor -1\right)  \tag{31}
\label{31}
\end{equation}%
(call $\varepsilon :=nb-\left \lfloor nb\right \rfloor $, $0\leq \varepsilon
<1 $) 
\begin{equation*}
=\Phi \left( \varepsilon -1\right) =\Phi \left( 1-\varepsilon \right) \geq
\Phi \left( 1\right) >0.
\end{equation*}%
Therefore 
\begin{equation}
\underset{n\rightarrow \rightarrow \infty }{\lim }\left( 1-\sum
\limits_{k=\left \lceil na\right \rceil }^{\left \lfloor nb\right \rfloor
}\Phi \left( nb-k\right) \right) >0.  \tag{32}  \label{32}
\end{equation}%
Similarly, it holds 
\begin{equation*}
1-\sum \limits_{k=\left \lceil na\right \rceil }^{\left \lfloor nb\right
\rfloor }\Phi \left( na-k\right) =\sum \limits_{k=-\infty }^{\left \lceil
na\right \rceil -1}\Phi \left( na-k\right) +\sum \limits_{k=\left \lfloor
nb\right \rfloor +1}^{\infty }\Phi \left( na-k\right)
\end{equation*}%
\begin{equation}
>\Phi \left( na-\left \lceil na\right \rceil +1\right)  \tag{33}  \label{33}
\end{equation}%
(call $\eta :=\left \lceil na\right \rceil -na,$ $0\leq \eta <1$) 
\begin{equation*}
=\Phi \left( 1-\eta \right) \geq \Phi \left( 1\right) >0.
\end{equation*}%
Therefore again 
\begin{equation}
\underset{n\rightarrow \infty }{\lim }\left( 1-\sum \limits_{k=\left \lceil
na\right \rceil }^{\left \lfloor nb\right \rfloor }\Phi \left( na-k\right)
\right) >0.  \tag{34}  \label{34}
\end{equation}%
Here we find that 
\begin{equation}
\underset{n\rightarrow \infty }{\lim }\sum \limits_{k=\left \lceil na\right
\rceil }^{\left \lfloor nb\right \rfloor }\Phi \left( nx-k\right) \neq 1,%
\text{ \ for at least some }x\in \left[ a,b\right] .  \tag{35}  \label{35}
\end{equation}
\end{remark}

\begin{note}
\label{n7}Let $\left[ a,b\right] \subset \mathbb{R}$. For large enough $n$
we always obtain $\left \lceil na\right \rceil \leq \left \lfloor
nb\right
\rfloor $. Also $a\leq \frac{k}{n}\leq b$, iff $\left \lceil
na\right \rceil \leq k\leq \left \lfloor nb\right \rfloor $.

In general it holds (by $\sum \limits_{i=-\infty }^{\infty }\Phi \left(
x-i\right) =1$, $\forall $ $x\in \mathbb{R}$) that 
\begin{equation}
\sum \limits_{k=\left \lceil na\right \rceil }^{\left \lfloor nb\right
\rfloor }\Phi \left( nx-k\right) \leq 1.  \tag{36}  \label{36}
\end{equation}
\end{note}

Let $\left( X,\left \Vert \cdot \right \Vert \right) $ be a Banach space.

\begin{definition}
\label{d8}Let $f\in C\left( \left[ a,b\right] ,X\right) $ and $n\in \mathbb{N%
}:\left \lceil na\right \rceil \leq \left \lfloor nb\right \rfloor $. We
introduce and define the $X$-valued linear neural network operators 
\begin{equation}
A_{n}\left( f,x\right) :=\frac{\sum \limits_{k=\left \lceil na\right \rceil
}^{\left \lfloor nb\right \rfloor }f\left( \frac{k}{n}\right) \Phi \left(
nx-k\right) }{\sum \limits_{k=\left \lceil na\right \rceil }^{\left \lfloor
nb\right \rfloor }\Phi \left( nx-k\right) }\text{, \ }x\in \left[ a,b\right]
.  \tag{37}  \label{37}
\end{equation}
\end{definition}

Clearly here $A_{n}\left( f,x\right) \in C\left( \left[ a,b\right] ,X\right) 
$.

For convenience we use the same $A_{n}$ for real valued functions when
needed. We study here the pointwise and uniform convergence of $A_{n}\left(
f,x\right) $ to $f\left( x\right) $ with rates.

For convenience, also we call 
\begin{equation}
A_{n}^{\ast }\left( f,x\right) :=\sum \limits_{k=\left \lceil na\right
\rceil }^{\left \lfloor nb\right \rfloor }f\left( \frac{k}{n}\right) \Phi
\left( nx-k\right) \text{, }  \tag{38}  \label{38}
\end{equation}%
(similarly, $A_{n}^{\ast }$ can be defined for real valued functions) that
is 
\begin{equation}
A_{n}\left( f,x\right) :=\frac{A_{n}^{\ast }\left( f,x\right) }{\sum
\limits_{k=\left \lceil na\right \rceil }^{\left \lfloor nb\right \rfloor
}\Phi \left( nx-k\right) }.  \tag{39}  \label{39}
\end{equation}%
So that 
\begin{equation*}
A_{n}\left( f,x\right) -f\left( x\right) =\frac{A_{n}^{\ast }\left(
f,x\right) }{\sum \limits_{k=\left \lceil na\right \rceil }^{\left \lfloor
nb\right \rfloor }\Phi \left( nx-k\right) }-f\left( x\right) =
\end{equation*}%
\begin{equation}
\frac{A_{n}^{\ast }\left( f,x\right) -f\left( x\right) \left( \sum
\limits_{k=\left \lceil na\right \rceil }^{\left \lfloor nb\right \rfloor
}\Phi \left( nx-k\right) \right) }{\sum \limits_{k=\left \lceil na\right
\rceil }^{\left \lfloor nb\right \rfloor }\Phi \left( nx-k\right) }. 
\tag{40}  \label{40}
\end{equation}%
Consequently, we derive that 
\begin{equation*}
\left \Vert A_{n}\left( f,x\right) -f\left( x\right) \right \Vert \leq
2\left( \sqrt[2m]{1+4^{m}}\right) \left \Vert A_{n}^{\ast }\left( f,x\right)
-f\left( x\right) \left( \sum \limits_{k=\left \lceil na\right \rceil
}^{\left \lfloor nb\right \rfloor }\Phi \left( nx-k\right) \right) \right
\Vert =
\end{equation*}%
\begin{equation}
2\left( \sqrt[2m]{1+4^{m}}\right) \left \Vert \sum \limits_{k=\left \lceil
na\right \rceil }^{\left \lfloor nb\right \rfloor }\left( f\left( \frac{k}{n}%
\right) -f\left( x\right) \right) \Phi \left( nx-k\right) \right \Vert . 
\tag{41}  \label{41}
\end{equation}

We will estimate the right and hand side of (\ref{41}).

For that we need, for $f\in C\left( \left[ a,b\right] ,X\right) $ the first
modulus of continuity 
\begin{equation*}
\omega _{1}\left( f,\delta \right) :=\underset{%
\begin{array}{c}
x,y\in \left[ a,b\right] \\ 
\left \vert x-y\right \vert \leq \delta%
\end{array}%
}{\sup }\left \Vert f\left( x\right) -f\left( y\right) \right \Vert ,\text{
\ }\delta >0.
\end{equation*}%
Similarly, it is defined $\omega _{1}$ for $f\in C_{uB}\left( \mathbb{R}%
,X\right) $ (uniformly continuous and bounded functions from $\mathbb{R}$
into $X$), for $f\in C_{B}\left( \mathbb{R},X\right) $ (continuous and
bounded $X$-valued), and for $f\in C_{u}\left( \mathbb{R},X\right) $
(uniformly continuous).

The fact $f\in C\left( \left[ a,b\right] ,X\right) $ or $f\in C_{u}\left( 
\mathbb{R},X\right) $, is equivalent to $\underset{\delta \rightarrow 0}{%
\lim }\omega _{1}\left( f,\delta \right) =0$, see \cite{11}.

We make

\begin{definition}
\label{d9}When $f\in C_{uB}\left( \mathbb{R},X\right) $, or $f\in
C_{B}\left( \mathbb{R},X\right) $, we define 
\begin{equation}
\overline{A_{n}}\left( f,x\right) :=\sum \limits_{k=-\infty }^{\infty
}f\left( \frac{k}{n}\right) \Phi \left( nx-k\right) ,  \tag{42}  \label{42}
\end{equation}%
$n\in \mathbb{N}$, $x\in \mathbb{R}$,

the $X$-valued quasi-interpolation neural network operator.
\end{definition}

We make

\begin{remark}
\label{r10}We have that 
\begin{equation}
\left \Vert f\left( \frac{k}{n}\right) \right \Vert \leq \left \Vert f\right
\Vert _{\infty ,\mathbb{R}}<+\infty ,  \tag{43}  \label{43}
\end{equation}%
and 
\begin{equation}
\left \Vert f\left( \frac{k}{n}\right) \right \Vert \Phi \left( nx-k\right)
\leq \left \Vert f\right \Vert _{\infty ,\mathbb{R}}\Phi \left( nx-k\right) 
\tag{44}  \label{44}
\end{equation}%
and 
\begin{equation}
\sum \limits_{k=-\lambda }^{\lambda }\left \Vert f\left( \frac{k}{n}\right)
\right \Vert \Phi \left( nx-k\right) \leq \left \Vert f\right \Vert _{\infty
,\mathbb{R}}\left( \sum \limits_{k=-\lambda }^{\lambda }\Phi \left(
nx-k\right) \right) ,  \tag{45}  \label{45}
\end{equation}%
and finally 
\begin{equation}
\sum \limits_{k=-\infty }^{\infty }\left \Vert f\left( \frac{k}{n}\right)
\right \Vert \Phi \left( nx-k\right) \leq \left \Vert f\right \Vert _{\infty
,\mathbb{R}},  \tag{46}  \label{46}
\end{equation}%
a convergent in $\mathbb{R}$ series.

So, the series $\sum \limits_{k=-\infty }^{\infty }\left \Vert f\left( \frac{%
k}{n}\right) \right \Vert \Phi \left( nx-k\right) $ is absolutely convergent
in $X$, hence it is convergent in $X$ and $\overline{A_{n}}\left( f,x\right)
\in X$. We denote by $\left \Vert f\right \Vert _{\infty }:=\underset{x\in %
\left[ a,b\right] }{\sup }\left \Vert f\left( x\right) \right \Vert $, for $%
f\in C\left( \left[ a,b\right] ,X\right) $, similarly it is defined for $%
f\in C_{B}\left( \mathbb{R},X\right) .$
\end{remark}

\section{Main Results}

We present a set of $X$-valued neural network approximations to a function
given with rates.

\begin{theorem}
\label{t11.}Let $f\in C\left( \left[ a,b\right] ,X\right) $, $0<\alpha <1$, $%
n\in \mathbb{N}:n^{1-\alpha }>2$, $x\in \left[ a,b\right] $, $m\in \mathbb{N}%
.$ Then

i) 
\begin{equation}
\left \Vert A_{n}\left( f,x\right) -f\left( x\right) \right \Vert \leq
\left( \sqrt[2m]{1+4^{m}}\right) \left[ 2\omega _{1}\left( f,\frac{1}{%
n^{\alpha }}\right) +\frac{\left \Vert f\right \Vert _{\infty }}{m\left(
n^{1-\alpha }-2\right) ^{2m}}\right] =:\lambda _{1},  \tag{47}  \label{47}
\end{equation}%
and

ii) 
\begin{equation}
\left \Vert A_{n}\left( f\right) -f\right \Vert _{\infty }\leq \lambda _{1}.
\tag{48}  \label{48}
\end{equation}%
We get that $\underset{n\rightarrow \infty }{\lim }A_{n}\left( f\right) =f$,
pointwise and uniformly.
\end{theorem}

\begin{proof}
We see that 
\begin{equation*}
\left \Vert \sum_{k=\left \lceil na\right \rceil }^{\left \lfloor nb\right
\rfloor }\left( f\left( \frac{k}{n}\right) -f\left( x\right) \right) \Phi
\left( nx-k\right) \right \Vert \leq
\end{equation*}%
\begin{equation*}
\sum_{k=\left \lceil na\right \rceil }^{\left \lfloor nb\right \rfloor
}\left \Vert f\left( \frac{k}{n}\right) -f\left( x\right) \right \Vert \Phi
\left( nx-k\right) =
\end{equation*}%
\begin{equation*}
\sum_{\left \{ 
\begin{array}{l}
k=\left \lceil na\right \rceil \\ 
:\left \vert \frac{k}{n}-x\right \vert \leq \frac{1}{n^{\alpha }}%
\end{array}%
\right. }^{\left \lfloor nb\right \rfloor }\left \Vert f\left( \frac{k}{n}%
\right) -f\left( x\right) \right \Vert \Phi \left( nx-k\right) +
\end{equation*}%
\begin{equation}
\sum_{\left \{ 
\begin{array}{l}
k=\left \lceil na\right \rceil \\ 
:\left \vert \frac{k}{n}-x\right \vert >\frac{1}{n^{\alpha }}%
\end{array}%
\right. }^{\left \lfloor nb\right \rfloor }\left \Vert f\left( \frac{k}{n}%
\right) -f\left( x\right) \right \Vert \Phi \left( nx-k\right) \leq  \tag{49}
\label{49}
\end{equation}%
\begin{equation*}
\sum_{\left \{ 
\begin{array}{l}
k=\left \lceil na\right \rceil \\ 
:\left \vert \frac{k}{n}-x\right \vert \leq \frac{1}{n^{\alpha }}%
\end{array}%
\right. }^{\left \lfloor nb\right \rfloor }\omega _{1}\left( f,\left \vert 
\frac{k}{n}-x\right \vert \right) \Phi \left( nx-k\right) +
\end{equation*}%
\begin{equation*}
2\left \Vert f\right \Vert _{\infty }\sum_{\left \{ 
\begin{array}{l}
k=-\infty \\ 
:\left \vert k-nx\right \vert >n^{1-\alpha }%
\end{array}%
\right. }^{\infty }\Phi \left( nx-k\right) \leq
\end{equation*}%
\begin{equation*}
\omega _{1}\left( f,\frac{1}{n^{\alpha }}\right) \sum_{\left \{ 
\begin{array}{l}
k=-\infty \\ 
:\left \vert \frac{k}{n}-x\right \vert \leq \frac{1}{n^{\alpha }}%
\end{array}%
\right. }^{\infty }\Phi \left( nx-k\right) +
\end{equation*}%
\begin{equation*}
2\left \Vert f\right \Vert _{\infty }\sum_{\left \{ 
\begin{array}{l}
k=-\infty \\ 
:\left \vert k-nx\right \vert >n^{1-\alpha }%
\end{array}%
\right. }^{\infty }\Phi \left( nx-k\right) \underset{\text{(by Theorem \ref%
{t4})}}{\leq }
\end{equation*}%
\begin{equation*}
\omega _{1}\left( f,\frac{1}{n^{\alpha }}\right) +\frac{\left \Vert f\right
\Vert _{\infty }}{2m\left( n^{1-\alpha }-2\right) ^{2m}}.
\end{equation*}%
That is 
\begin{equation*}
\left \Vert \sum_{k=\left \lceil na\right \rceil }^{\left \lfloor nb\right
\rfloor }\left( f\left( \frac{k}{n}\right) -f\left( x\right) \right) \Phi
\left( nx-k\right) \right \Vert \leq
\end{equation*}%
\begin{equation}
\omega _{1}\left( f,\frac{1}{n^{\alpha }}\right) +\frac{\left \Vert f\right
\Vert _{\infty }}{2m\left( n^{1-\alpha }-2\right) ^{2m}}.  \tag{50}
\label{50}
\end{equation}%
Using (\ref{41}) we derive (\ref{47}).
\end{proof}

It follows

\begin{theorem}
\label{t12.}Let $f\in C_{B}\left( \mathbb{R},X\right) $, $0<\alpha <1$, $%
n\in \mathbb{N}:n^{1-\alpha }>2$, $x\in \mathbb{R},$ $m\in \mathbb{N}.$ Then

i) 
\begin{equation}
\left \Vert \overline{A_{n}}\left( f,x\right) -f\left( x\right) \right \Vert
\leq \omega _{1}\left( f,\frac{1}{n^{\alpha }}\right) +\frac{\left \Vert
f\right \Vert _{\infty }}{2m\left( n^{1-\alpha }-2\right) ^{2m}}=:\lambda
_{2},  \tag{51}  \label{51}
\end{equation}%
and

ii) 
\begin{equation}
\left \Vert \overline{A_{n}}\left( f\right) -f\right \Vert _{\infty }\leq
\lambda _{2}.  \tag{52}  \label{52}
\end{equation}%
For $f\in C_{uB}\left( \mathbb{R},X\right) $ we get $\underset{n\rightarrow
\infty }{\lim }\overline{A_{n}}\left( f\right) =f$, pointwise and uniformly.
\end{theorem}

\begin{proof}
We observe that 
\begin{equation*}
\left \Vert \overline{A_{n}}\left( f,x\right) -f\left( x\right) \right \Vert
=\left \Vert \sum_{k=-\infty }^{\infty }f\left( \frac{k}{n}\right) \Phi
\left( nx-k\right) -f\left( x\right) \sum_{k=-\infty }^{\infty }\Phi \left(
nx-k\right) \right \Vert =
\end{equation*}%
\begin{equation*}
\left \Vert \sum_{k=-\infty }^{\infty }\left( f\left( \frac{k}{n}\right)
-f\left( x\right) \right) \Phi \left( nx-k\right) \right \Vert \leq
\end{equation*}%
\begin{equation}
\sum_{k=-\infty }^{\infty }\left \Vert f\left( \frac{k}{n}\right) -f\left(
x\right) \right \Vert \Phi \left( nx-k\right) =  \tag{53}  \label{53}
\end{equation}%
\begin{equation*}
\sum_{\left \{ 
\begin{array}{l}
k=-\infty \\ 
:\left \vert \frac{k}{n}-x\right \vert \leq \frac{1}{n^{\alpha }}%
\end{array}%
\right. }^{\infty }\left \Vert f\left( \frac{k}{n}\right) -f\left( x\right)
\right \Vert \Phi \left( nx-k\right) +
\end{equation*}%
\begin{equation*}
\sum_{\left \{ 
\begin{array}{l}
k=-\infty \\ 
:\left \vert \frac{k}{n}-x\right \vert >\frac{1}{n^{\alpha }}%
\end{array}%
\right. }^{\infty }\left \Vert f\left( \frac{k}{n}\right) -f\left( x\right)
\right \Vert \Phi \left( nx-k\right) \leq
\end{equation*}%
\begin{equation*}
\sum_{\left \{ 
\begin{array}{l}
k=-\infty \\ 
:\left \vert \frac{k}{n}-x\right \vert \leq \frac{1}{n^{\alpha }}%
\end{array}%
\right. }^{\infty }\omega _{1}\left( f,\left \vert \frac{k}{n}-x\right \vert
\right) \Phi \left( nx-k\right) +
\end{equation*}%
\begin{equation*}
2\left \Vert f\right \Vert _{\infty }\sum_{\left \{ 
\begin{array}{l}
k=-\infty \\ 
:\left \vert \frac{k}{n}-x\right \vert >\frac{1}{n^{\alpha }}%
\end{array}%
\right. }^{\infty }\Phi \left( nx-k\right) \leq
\end{equation*}%
\begin{equation}
\omega _{1}\left( f,\frac{1}{n^{\alpha }}\right) \sum_{\left \{ 
\begin{array}{l}
k=-\infty \\ 
:\left \vert \frac{k}{n}-x\right \vert \leq \frac{1}{n^{\alpha }}%
\end{array}%
\right. }^{\infty }\Phi \left( nx-k\right) +\frac{2\left \Vert f\right \Vert
_{\infty }}{4m\left( n^{1-\alpha }-2\right) ^{2m}}\leq  \tag{54}  \label{54}
\end{equation}%
\begin{equation*}
\omega _{1}\left( f,\frac{1}{n^{\alpha }}\right) +\frac{\left \Vert f\right
\Vert _{\infty }}{2m\left( n^{1-\alpha }-2\right) ^{2m}},
\end{equation*}%
proving the claim.
\end{proof}

We need the $X$-valued Taylor's formula in an appropiate form:

\begin{theorem}
\label{t13.}(\cite{10}, \cite{12}) Let $N\in \mathbb{N}$, and $f\in
C^{N}\left( \left[ a,b\right] ,X\right) $, where $\left[ a,b\right] \subset 
\mathbb{R}$ and $X$ is a Banach space. Let any $x,y\in \left[ a,b\right] $.
Then 
\begin{equation}
f\left( x\right) =\sum_{i=0}^{N}\frac{\left( x-y\right) ^{i}}{i!}f^{\left(
i\right) }\left( y\right) +\frac{1}{\left( N-1\right) !}\int_{y}^{x}\left(
x-t\right) ^{N-1}\left( f^{\left( N\right) }\left( t\right) -f^{\left(
N\right) }\left( y\right) \right) dt.  \tag{55}  \label{55}
\end{equation}
\end{theorem}

The derivatives $f^{\left( i\right) }$, $i\in \mathbb{N}$, are defined like
the numerical ones, see \cite{20}, p. 83. The integral $\int_{y}^{x}$ in (%
\ref{55}) is of Bochner type, see \cite{18}.

By \cite{12}, \cite{16} we have that: if $f\in C\left( \left[ a,b\right]
,X\right) $, then $f\in L_{\infty }\left( \left[ a,b\right] ,X\right) $ and $%
f\in L_{1}\left( \left[ a,b\right] ,X\right) .$

In the next we discuss high order neural network $X$-valued approximation by
using the smoothness of $f$.

\begin{theorem}
\label{t14.}Let $f\in C^{N}\left( \left[ a,b\right] ,X\right) $, $n,N,m\in 
\mathbb{N}$, $0<\alpha <1$, $x\in \left[ a,b\right] $ and $n^{1-\alpha }>2$.
Then

i) 
\begin{equation}
\left \Vert A_{n}\left( f,x\right) -f\left( x\right) \right \Vert \leq
\left( \sqrt[2m]{1+4^{m}}\right) \left \{ \sum_{j=1}^{N}\frac{\left \Vert
f^{\left( j\right) }\left( x\right) \right \Vert }{j!}\left[ \frac{2}{%
n^{\alpha j}}+\frac{\left( b-a\right) ^{j}}{2m\left( n^{1-\alpha }-2\right)
^{2m}}\right] +\right.  \tag{56}  \label{56}
\end{equation}%
\begin{equation*}
\left. \left[ \omega _{1}\left( f^{\left( N\right) },\frac{1}{n^{\alpha }}%
\right) \frac{2}{n^{\alpha N}N!}+\frac{\left \Vert f^{\left( N\right)
}\right \Vert _{\infty }\left( b-a\right) ^{N}}{N!m\left( n^{1-\alpha
}-2\right) ^{2m}}\right] \right \} ,
\end{equation*}

ii) assume further $f^{\left( j\right) }\left( x_{0}\right) =0$, $j=1,...,N,$
for some $x_{0}\in \left[ a,b\right] $, it holds 
\begin{equation}
\left \Vert A_{n}\left( f,x_{0}\right) -f\left( x_{0}\right) \right \Vert
\leq \left( \sqrt[2m]{1+4^{m}}\right) \cdot  \tag{57}  \label{57}
\end{equation}%
\begin{equation*}
\left[ \omega _{1}\left( f^{\left( N\right) },\frac{1}{n^{\alpha }}\right) 
\frac{2}{n^{\alpha N}N!}+\frac{\left \Vert f^{\left( N\right) }\right \Vert
_{\infty }\left( b-a\right) ^{N}}{N!m\left( n^{1-\alpha }-2\right) ^{2m}}%
\right] ,
\end{equation*}%
and

iii) 
\begin{equation*}
\left \Vert A_{n}\left( f\right) -f\right \Vert _{\infty }\leq \left( \sqrt[%
2m]{1+4^{m}}\right) \left \{ \sum_{j=1}^{N}\frac{\left \Vert f^{\left(
j\right) }\right \Vert _{\infty }}{j!}\left[ \frac{2}{n^{\alpha j}}+\frac{%
\left( b-a\right) ^{j}}{2m\left( n^{1-\alpha }-2\right) ^{2m}}\right]
+\right.
\end{equation*}%
\begin{equation}
\left. \left[ \omega _{1}\left( f^{\left( N\right) },\frac{1}{n^{\alpha }}%
\right) \frac{2}{n^{\alpha N}N!}+\frac{\left \Vert f^{\left( N\right)
}\right \Vert _{\infty }\left( b-a\right) ^{N}}{N!m\left( n^{1-\alpha
}-2\right) ^{2m}}\right] \right \} .  \tag{58}  \label{58}
\end{equation}%
We derive that $\underset{n\rightarrow \infty }{\lim }A_{n}\left( f\right)
=f $, pointwise and uniformly.
\end{theorem}

\begin{proof}
Next we apply the $X$-valued Taylor's formula with Bochner integral
remainder (\ref{55}). We have (here $\frac{k}{n},x\in \left[ a,b\right] $) 
\begin{equation}
f\left( \frac{k}{n}\right) =\sum_{j=0}^{N}\frac{f^{\left( j\right) }\left(
x\right) }{j!}\left( \frac{k}{n}-x\right) ^{j}+\int_{x}^{\frac{k}{n}}\left(
f^{\left( N\right) }\left( t\right) -f^{\left( N\right) }\left( x\right)
\right) \frac{\left( \frac{k}{n}-t\right) ^{N-1}}{\left( N-1\right) !}dt. 
\tag{59}  \label{59}
\end{equation}%
Then 
\begin{equation}
f\left( \frac{k}{n}\right) \Phi \left( nx-k\right) =\sum_{j=0}^{N}\frac{%
f^{\left( j\right) }\left( x\right) }{j!}\Phi \left( nx-k\right) \left( 
\frac{k}{n}-x\right) ^{j}+  \tag{60}  \label{60}
\end{equation}%
\begin{equation*}
\Phi \left( nx-k\right) \int_{x}^{\frac{k}{n}}\left( f^{\left( N\right)
}\left( t\right) -f^{\left( N\right) }\left( x\right) \right) \frac{\left( 
\frac{k}{n}-t\right) ^{N-1}}{\left( N-1\right) !}dt.
\end{equation*}%
Hence 
\begin{equation}
\sum_{k=\left \lceil na\right \rceil }^{\left \lfloor nb\right \rfloor
}f\left( \frac{k}{n}\right) \Phi \left( nx-k\right) -f\left( x\right)
\sum_{k=\left \lceil na\right \rceil }^{\left \lfloor nb\right \rfloor }\Phi
\left( nx-k\right) =  \tag{61}  \label{61}
\end{equation}%
\begin{equation*}
\sum_{j=1}^{N}\frac{f^{\left( j\right) }\left( x\right) }{j!}\sum_{k=\left
\lceil na\right \rceil }^{\left \lfloor nb\right \rfloor }\Phi \left(
nx-k\right) \left( \frac{k}{n}-x\right) ^{j}+
\end{equation*}%
\begin{equation*}
\sum_{k=\left \lceil na\right \rceil }^{\left \lfloor nb\right \rfloor }\Phi
\left( nx-k\right) \int_{x}^{\frac{k}{n}}\left( f^{\left( N\right) }\left(
t\right) -f^{\left( N\right) }\left( x\right) \right) \frac{\left( \frac{k}{n%
}-t\right) ^{N-1}}{\left( N-1\right) !}dt.
\end{equation*}%
Thus 
\begin{equation*}
A_{n}^{\ast }\left( f,x\right) -f\left( x\right) \left( \sum_{k=\left \lceil
na\right \rceil }^{\left \lfloor nb\right \rfloor }\Phi \left( nx-k\right)
\right) =
\end{equation*}%
\begin{equation}
\sum_{j=1}^{N}\frac{f^{\left( j\right) }\left( x\right) }{j!}A_{n}^{\ast
}\left( \left( \cdot -x\right) ^{j}\right) +\Lambda _{n}\left( x\right) , 
\tag{62}  \label{62}
\end{equation}%
where 
\begin{equation}
\Lambda _{n}\left( x\right) :=\sum_{k=\left \lceil na\right \rceil }^{\left
\lfloor nb\right \rfloor }\Phi \left( nx-k\right) \int_{x}^{\frac{k}{n}%
}\left( f^{\left( N\right) }\left( t\right) -f^{\left( N\right) }\left(
x\right) \right) \frac{\left( \frac{k}{n}-t\right) ^{N-1}}{\left( N-1\right)
!}dt.  \tag{63}  \label{63}
\end{equation}%
We assume that $b-a>\frac{1}{n^{\alpha }}$, which is always the case for
large enough $n\in \mathbb{N}$, that is when $n>\left \lceil \left(
b-a\right) ^{-\frac{1}{\alpha }}\right \rceil .$

Thus $\left| \frac{k}{n}-x\right| \leq \frac{1}{n^{\alpha }}$ or $\left| 
\frac{k}{n}-x\right| >\frac{1}{n^{\alpha }}.$

Let 
\begin{equation}
\gamma :=\int_{x}^{\frac{k}{n}}\left( f^{\left( N\right) }\left( t\right)
-f^{\left( N\right) }\left( x\right) \right) \frac{\left( \frac{k}{n}%
-t\right) ^{N-1}}{\left( N-1\right) !}dt,  \tag{64}  \label{64}
\end{equation}%
in the case of $\left \vert \frac{k}{n}-x\right \vert \leq \frac{1}{%
n^{\alpha }}$, we find that 
\begin{equation}
\left \Vert \gamma \right \Vert \leq \omega _{1}\left( f^{\left( N\right) },%
\frac{1}{n^{\alpha }}\right) \frac{1}{n^{\alpha N}N!}  \tag{65}  \label{65}
\end{equation}%
for $x\leq \frac{k}{n}$ or $x\geq \frac{k}{n}.$

We prove it next.

i) Indeed, for the case of $x\leq \frac{k}{n}$, we have 
\begin{equation*}
\left \Vert \gamma \right \Vert =\left \Vert \int_{x}^{\frac{k}{n}}\left(
f^{\left( N\right) }\left( t\right) -f^{\left( N\right) }\left( x\right)
\right) \frac{\left( \frac{k}{n}-t\right) ^{N-1}}{\left( N-1\right) !}%
dt\right \Vert \leq
\end{equation*}%
\begin{equation*}
\int_{x}^{\frac{k}{n}}\left \Vert f^{\left( N\right) }\left( t\right)
-f^{\left( N\right) }\left( x\right) \right \Vert \frac{\left( \frac{k}{n}%
-t\right) ^{N-1}}{\left( N-1\right) !}dt\leq
\end{equation*}%
\begin{equation}
\int_{x}^{\frac{k}{n}}\omega _{1}\left( f^{\left( N\right) },\left \vert
t-x\right \vert \right) \frac{\left( \frac{k}{n}-t\right) ^{N-1}}{\left(
N-1\right) !}dt\leq \omega _{1}\left( f^{\left( N\right) },\frac{1}{%
n^{\alpha }}\right) \int_{x}^{\frac{k}{n}}\frac{\left( \frac{k}{n}-t\right)
^{N-1}}{\left( N-1\right) !}dt=  \tag{66}  \label{66}
\end{equation}%
\begin{equation*}
\omega _{1}\left( f^{\left( N\right) },\frac{1}{n^{\alpha }}\right) \frac{%
\left( \frac{k}{n}-x\right) ^{N}}{N!}\leq \omega _{1}\left( f^{\left(
N\right) },\frac{1}{n^{\alpha }}\right) \frac{1}{n^{\alpha N}N!}.
\end{equation*}

ii) for the case of $x>\frac{k}{n}$, we have 
\begin{equation*}
\left \Vert \gamma \right \Vert =\left \Vert \int_{x}^{\frac{k}{n}}\left(
f^{\left( N\right) }\left( t\right) -f^{\left( N\right) }\left( x\right)
\right) \frac{\left( \frac{k}{n}-t\right) ^{N-1}}{\left( N-1\right) !}%
dt\right \Vert =
\end{equation*}%
\begin{equation*}
\left \Vert \int_{\frac{k}{n}}^{x}\left( f^{\left( N\right) }\left( t\right)
-f^{\left( N\right) }\left( x\right) \right) \frac{\left( t-\frac{k}{n}%
\right) ^{N-1}}{\left( N-1\right) !}dt\right \Vert \leq
\end{equation*}%
\begin{equation}
\int_{\frac{k}{n}}^{x}\left \Vert f^{\left( N\right) }\left( t\right)
-f^{\left( N\right) }\left( x\right) \right \Vert \frac{\left( t-\frac{k}{n}%
\right) ^{N-1}}{\left( N-1\right) !}dt\leq  \tag{67}  \label{67}
\end{equation}%
\begin{equation*}
\int_{\frac{k}{n}}^{x}\omega _{1}\left( f^{\left( N\right) },\left \vert
t-x\right \vert \right) \frac{\left( t-\frac{k}{n}\right) ^{N-1}}{\left(
N-1\right) !}dt\leq \omega _{1}\left( f^{\left( N\right) },\frac{1}{%
n^{\alpha }}\right) \int_{\frac{k}{n}}^{x}\frac{\left( t-\frac{k}{n}\right)
^{N-1}}{\left( N-1\right) !}dt=
\end{equation*}%
\begin{equation*}
\omega _{1}\left( f^{\left( N\right) },\frac{1}{n^{\alpha }}\right) \frac{%
\left( x-\frac{k}{n}\right) ^{N}}{N!}\leq \omega _{1}\left( f^{\left(
N\right) },\frac{1}{n^{\alpha }}\right) \frac{1}{n^{\alpha N}N!}.
\end{equation*}%
We have proved (\ref{65}).

We treat again $\gamma $, see (\ref{64}), but differently:

Notice also for $x\leq \frac{k}{n}$ that 
\begin{equation*}
\left \Vert \int_{x}^{\frac{k}{n}}\left( f^{\left( N\right) }\left( t\right)
-f^{\left( N\right) }\left( x\right) \right) \frac{\left( \frac{k}{n}%
-t\right) ^{N-1}}{\left( N-1\right) !}dt\right \Vert \leq
\end{equation*}%
\begin{equation}
\int_{x}^{\frac{k}{n}}\left \Vert f^{\left( N\right) }\left( t\right)
-f^{\left( N\right) }\left( x\right) \right \Vert \frac{\left( \frac{k}{n}%
-t\right) ^{N-1}}{\left( N-1\right) !}dt\leq  \tag{68}  \label{68}
\end{equation}%
\begin{equation*}
2\left \Vert f^{\left( N\right) }\right \Vert _{\infty }\int_{x}^{\frac{k}{n}%
}\frac{\left( \frac{k}{n}-t\right) ^{N-1}}{\left( N-1\right) !}dt=2\left
\Vert f^{\left( N\right) }\right \Vert _{\infty }\frac{\left( \frac{k}{n}%
-x\right) ^{N}}{N!}
\end{equation*}%
\begin{equation*}
\leq 2\left \Vert f^{\left( N\right) }\right \Vert _{\infty }\frac{\left(
b-a\right) ^{N}}{N!}.
\end{equation*}%
Next assume $\frac{k}{n}\leq x$, then 
\begin{equation*}
\left \Vert \int_{x}^{\frac{k}{n}}\left( f^{\left( N\right) }\left( t\right)
-f^{\left( N\right) }\left( x\right) \right) \frac{\left( \frac{k}{n}%
-t\right) ^{N-1}}{\left( N-1\right) !}dt\right \Vert =
\end{equation*}%
\begin{equation*}
\left \Vert \int_{\frac{k}{n}}^{x}\left( f^{\left( N\right) }\left( t\right)
-f^{\left( N\right) }\left( x\right) \right) \frac{\left( t-\frac{k}{n}%
\right) ^{N-1}}{\left( N-1\right) !}dt\right \Vert \leq
\end{equation*}%
\begin{equation}
\int_{\frac{k}{n}}^{x}\left \Vert f^{\left( N\right) }\left( t\right)
-f^{\left( N\right) }\left( x\right) \right \Vert \frac{\left( t-\frac{k}{n}%
\right) ^{N-1}}{\left( N-1\right) !}dt\leq  \tag{69}  \label{69}
\end{equation}%
\begin{equation*}
2\left \Vert f^{\left( N\right) }\right \Vert _{\infty }\int_{\frac{k}{n}%
}^{x}\frac{\left( t-\frac{k}{n}\right) ^{N-1}}{\left( N-1\right) !}dt=2\left
\Vert f^{\left( N\right) }\right \Vert _{\infty }\frac{\left( x-\frac{k}{n}%
\right) ^{N}}{N!}
\end{equation*}%
\begin{equation*}
\leq 2\left \Vert f^{\left( N\right) }\right \Vert _{\infty }\frac{\left(
b-a\right) ^{N}}{N!}.
\end{equation*}%
Thus 
\begin{equation}
\left \Vert \gamma \right \Vert \leq 2\left \Vert f^{\left( N\right) }\right
\Vert _{\infty }\frac{\left( b-a\right) ^{N}}{N!}.  \tag{70}  \label{70}
\end{equation}%
in the two cases.

Therefore 
\begin{equation}
\Lambda _{n}\left( x\right) =\sum_{\left \{ 
\begin{array}{l}
k=\left \lceil na\right \rceil \\ 
\left \vert \frac{k}{n}-x\right \vert \leq \frac{1}{n^{\alpha }}%
\end{array}%
\right. }^{\left \lfloor nb\right \rfloor }\Phi \left( nx-k\right) \gamma
+\sum_{\left \{ 
\begin{array}{l}
k=\left \lceil na\right \rceil \\ 
\left \vert \frac{k}{n}-x\right \vert >\frac{1}{n^{\alpha }}%
\end{array}%
\right. }^{\left \lfloor nb\right \rfloor }\Phi \left( nx-k\right) \gamma . 
\tag{71}  \label{71}
\end{equation}%
Hence 
\begin{equation}
\left \Vert \Lambda _{n}\left( x\right) \right \Vert \leq \sum_{\left \{ 
\begin{array}{l}
k=\left \lceil na\right \rceil \\ 
\left \vert \frac{k}{n}-x\right \vert \leq \frac{1}{n^{\alpha }}%
\end{array}%
\right. }^{\left \lfloor nb\right \rfloor }\Phi \left( nx-k\right) \left(
\omega _{1}\left( f^{\left( N\right) },\frac{1}{n^{\alpha }}\right) \frac{1}{%
N!n^{\alpha N}}\right) +  \tag{72}  \label{72}
\end{equation}%
\begin{equation*}
\left( \sum_{\left \{ 
\begin{array}{l}
k=\left \lceil na\right \rceil \\ 
\left \vert \frac{k}{n}-x\right \vert >\frac{1}{n^{\alpha }}%
\end{array}%
\right. }^{\left \lfloor nb\right \rfloor }\Phi \left( nx-k\right) \right)
2\left \Vert f^{\left( N\right) }\right \Vert _{\infty }\frac{\left(
b-a\right) ^{N}}{N!}\overset{\text{(\ref{19})}}{\leq }
\end{equation*}%
\begin{equation*}
\omega _{1}\left( f^{\left( N\right) },\frac{1}{n^{\alpha }}\right) \frac{1}{%
N!n^{\alpha N}}+\frac{1}{4m\left( n^{1-\alpha }-2\right) ^{2m}}2\left \Vert
f^{\left( N\right) }\right \Vert _{\infty }\frac{\left( b-a\right) ^{N}}{N!}=
\end{equation*}%
\begin{equation*}
\omega _{1}\left( f^{\left( N\right) },\frac{1}{n^{\alpha }}\right) \frac{1}{%
N!n^{\alpha N}}+\frac{\left \Vert f^{\left( N\right) }\right \Vert _{\infty
}\left( b-a\right) ^{N}}{N!2m\left( n^{1-\alpha }-2\right) ^{2m}}.
\end{equation*}%
That is 
\begin{equation}
\left \Vert \Lambda _{n}\left( x\right) \right \Vert \leq \frac{\omega
_{1}\left( f^{\left( N\right) },\frac{1}{n^{\alpha }}\right) }{N!n^{\alpha N}%
}+\frac{\left \Vert f^{\left( N\right) }\right \Vert _{\infty }\left(
b-a\right) ^{N}}{N!2m\left( n^{1-\alpha }-2\right) ^{2m}},  \tag{73}
\label{73}
\end{equation}%
$\forall $ $x\in \left[ a,b\right] .$

We further see that 
\begin{equation}
A_{n}^{\ast }\left( \left( \cdot -x\right) ^{j}\right) =\sum_{k=\left \lceil
na\right \rceil }^{\left \lfloor nb\right \rfloor }\Phi \left( nx-k\right)
\left( \frac{k}{n}-x\right) ^{j},  \tag{74}  \label{74}
\end{equation}%
where $A_{n}^{\ast }$ is defined similarly for real valued functions.

Therefore 
\begin{equation*}
\left \vert A_{n}^{\ast }\left( \left( \cdot -x\right) ^{j}\right) \right
\vert \leq \sum_{k=\left \lceil na\right \rceil }^{\left \lfloor nb\right
\rfloor }\Phi \left( nx-k\right) \left \vert \frac{k}{n}-x\right \vert ^{j}=
\end{equation*}%
\begin{equation*}
\sum_{\left \{ 
\begin{array}{l}
k=\left \lceil na\right \rceil \\ 
\left \vert \frac{k}{n}-x\right \vert \leq \frac{1}{n^{\alpha }}%
\end{array}%
\right. }^{\left \lfloor nb\right \rfloor }\Phi \left( nx-k\right) \left
\vert \frac{k}{n}-x\right \vert ^{j}+\sum_{\left \{ 
\begin{array}{l}
k=\left \lceil na\right \rceil \\ 
\left \vert \frac{k}{n}-x\right \vert >\frac{1}{n^{\alpha }}%
\end{array}%
\right. }^{\left \lfloor nb\right \rfloor }\Phi \left( nx-k\right) \left
\vert \frac{k}{n}-x\right \vert ^{j}\leq
\end{equation*}%
\begin{equation}
\frac{1}{n^{\alpha j}}+\left( b-a\right) ^{j}\frac{1}{4m\left( n^{1-\alpha
}-2\right) ^{2m}}.  \tag{75}  \label{75}
\end{equation}%
That is 
\begin{equation}
\left \vert A_{n}^{\ast }\left( \left( \cdot -x\right) ^{j}\right) \right
\vert \leq \frac{1}{n^{\alpha j}}+\left( b-a\right) ^{j}\frac{1}{4m\left(
n^{1-\alpha }-2\right) ^{2m}},  \tag{76}  \label{76}
\end{equation}%
for $j=1,...,N.$

Putting things together we have proved 
\begin{equation}
\left \Vert A_{n}^{\ast }\left( f,x\right) -f\left( x\right) \left(
\sum_{k=\left \lceil na\right \rceil }^{\left \lfloor nb\right \rfloor }\Phi
\left( nx-k\right) \right) \right \Vert \leq \sum_{j=1}^{N}\frac{\left \Vert
f^{\left( j\right) }\left( x\right) \right \Vert }{j!}  \tag{77}  \label{77}
\end{equation}%
\begin{equation*}
\left[ \frac{1}{n^{\alpha j}}+\frac{\left( b-a\right) ^{j}}{4m\left(
n^{1-\alpha }-2\right) ^{2m}}\right] +\left[ \omega _{1}\left( f^{\left(
N\right) },\frac{1}{n^{\alpha }}\right) \frac{1}{n^{\alpha N}N!}+\frac{\left
\Vert f^{\left( N\right) }\right \Vert _{\infty }\left( b-a\right) ^{N}}{%
N!2m\left( n^{1-\alpha }-2\right) ^{2m}}\right] ,
\end{equation*}%
that is establishing the theorem.
\end{proof}

All integrals from now on are of Bochner type \cite{18}.

We need

\begin{definition}
\label{d15}(\cite{12}) Let $\left[ a,b\right] \subset \mathbb{R}$, $X$ be a
Banach space, $\alpha >0$; $m=\left \lceil \alpha \right \rceil \in \mathbb{N%
}$, ($\left \lceil \cdot \right \rceil $ is the ceiling of the number), $f:%
\left[ a,b\right] \rightarrow X$. We assume that $f^{\left( m\right) }\in
L_{1}\left( \left[ a,b\right] ,X\right) $. We call the Caputo-Bochner left
fractional derivative of order $\alpha $: 
\begin{equation}
\left( D_{\ast a}^{\alpha }f\right) \left( x\right) :=\frac{1}{\Gamma \left(
m-\alpha \right) }\int_{a}^{x}\left( x-t\right) ^{m-\alpha -1}f^{\left(
m\right) }\left( t\right) dt,\text{ \ }\forall \text{ }x\in \left[ a,b\right]
.  \tag{78}  \label{78}
\end{equation}%
If $\alpha \in \mathbb{N}$, we set $D_{\ast a}^{\alpha }f:=f^{\left(
m\right) }$ the ordinary $X$-valued derivative (defined similar to numerical
one, see \cite{20}, p. 83), and also set $D_{\ast a}^{0}f:=f.$
\end{definition}

By \cite{12}, $\left( D_{\ast a}^{\alpha }f\right) \left( x\right) $ exists
almost everywhere in $x\in \left[ a,b\right] $ and $D_{\ast a}^{\alpha }f\in
L_{1}\left( \left[ a,b\right] ,X\right) $.

If $\left \| f^{\left( m\right) }\right \| _{L_{\infty }\left( \left[ a,b%
\right] ,X\right) }<\infty $, then by \cite{12}, $D_{\ast a}^{\alpha }f\in
C\left( \left[ a,b\right] ,X\right) ,$ hence $\left \| D_{\ast a}^{\alpha
}f\right \| \in C\left( \left[ a,b\right] \right) .$

We mention

\begin{lemma}
\label{l16}(\cite{11}) Let $\alpha >0$, $\alpha \notin \mathbb{N}$, $%
m=\left
\lceil \alpha \right \rceil $, $f\in C^{m-1}\left( \left[ a,b\right]
,X\right) $ and $f^{\left( m\right) }\in L_{\infty }\left( \left[ a,b\right]
,X\right) $. Then $D_{\ast a}^{\alpha }f\left( a\right) =0$.
\end{lemma}

We mention

\begin{definition}
\label{d17}(\cite{10}) Let $\left[ a,b\right] \subset \mathbb{R}$, $X$ be a
Banach space, $\alpha >0$, $m:=\left \lceil \alpha \right \rceil $. We
assume that $f^{\left( m\right) }\in L_{1}\left( \left[ a,b\right] ,X\right) 
$, where $f:\left[ a,b\right] \rightarrow X$. We call the Caputo-Bochner
right fractional derivative of order $\alpha $: 
\begin{equation}
\left( D_{b-}^{\alpha }f\right) \left( x\right) :=\frac{\left( -1\right) ^{m}%
}{\Gamma \left( m-\alpha \right) }\int_{x}^{b}\left( z-x\right) ^{m-\alpha
-1}f^{\left( m\right) }\left( z\right) dz,\text{ \ }\forall \text{ }x\in %
\left[ a,b\right] .  \tag{79}  \label{79}
\end{equation}%
We observe that $\left( D_{b-}^{m}f\right) \left( x\right) =\left( -1\right)
^{m}f^{\left( m\right) }\left( x\right) ,$ for $m\in \mathbb{N}$, and $%
\left( D_{b-}^{0}f\right) \left( x\right) =f\left( x\right) .$
\end{definition}

By \cite{10}, $\left( D_{b-}^{\alpha }f\right) \left( x\right) $ exists
almost everywhere on $\left[ a,b\right] $ and $\left( D_{b-}^{\alpha
}f\right) \in L_{1}\left( \left[ a,b\right] ,X\right) $.

If $\left \| f^{\left( m\right) }\right \| _{L_{\infty }\left( \left[ a,b%
\right] ,X\right) }<\infty $, and $\alpha \notin \mathbb{N},$ by \cite{10}, $%
D_{b-}^{\alpha }f\in C\left( \left[ a,b\right] ,X\right) ,$ hence $\left \|
D_{b-}^{\alpha }f\right \| \in C\left( \left[ a,b\right] \right) .$

We need

\begin{lemma}
\label{l18}(\cite{11}) Let $f\in C^{m-1}\left( \left[ a,b\right] ,X\right) $%
, $f^{\left( m\right) }\in L_{\infty }\left( \left[ a,b\right] ,X\right) $, $%
m=\left \lceil \alpha \right \rceil $, $\alpha >0$, $\alpha \notin \mathbb{N}
$. Then $D_{b-}^{\alpha }f\left( b\right) =0$.
\end{lemma}

We mention the left fractional Taylor formula

\begin{theorem}
\label{t19.}(\cite{12}) Let $m\in \mathbb{N}$ and $f\in C^{m}\left( \left[
a,b\right] ,X\right) ,$ where $\left[ a,b\right] \subset \mathbb{R}$ and $X$
is a Banach space, and let $\alpha >0:m=\left \lceil \alpha \right \rceil $.
Then 
\begin{equation}
f\left( x\right) =\sum_{i=0}^{m-1}\frac{\left( x-a\right) ^{i}}{i!}f^{\left(
i\right) }\left( a\right) +\frac{1}{\Gamma \left( \alpha \right) }%
\int_{a}^{x}\left( x-z\right) ^{\alpha -1}\left( D_{\ast a}^{\alpha
}f\right) \left( z\right) dz,  \tag{80}  \label{80}
\end{equation}%
$\forall $ $x\in \left[ a,b\right] .$
\end{theorem}

We also mention the right fractional Taylor formula

\begin{theorem}
\label{t20}(\cite{10}) Let $\left[ a,b\right] \subset \mathbb{R}$, $X$ be a
Banach space, $\alpha >0$, $m=\left \lceil \alpha \right \rceil $, $f\in
C^{m}\left( \left[ a,b\right] ,X\right) $. Then 
\begin{equation}
f\left( x\right) =\sum_{i=0}^{m-1}\frac{\left( x-b\right) ^{i}}{i!}f^{\left(
i\right) }\left( b\right) +\frac{1}{\Gamma \left( \alpha \right) }%
\int_{x}^{b}\left( z-x\right) ^{\alpha -1}\left( D_{b-}^{\alpha }f\right)
\left( z\right) dz,  \tag{81}  \label{81}
\end{equation}%
$\forall $ $x\in \left[ a,b\right] .$
\end{theorem}

\begin{convention}
\label{c21}We assume that 
\begin{equation}
D_{\ast x_{0}}^{\alpha }f\left( x\right) =0\text{, for }x<x_{0},  \tag{82}
\label{82}
\end{equation}%
and 
\begin{equation}
D_{x_{0}-}^{\alpha }f\left( x\right) =0\text{, for }x>x_{0},  \tag{83}
\label{83}
\end{equation}%
for all $x,x_{0}\in \left[ a,b\right] .$
\end{convention}

We mention

\begin{proposition}
\label{p22.}(\cite{11}) Let $f\in C^{n}\left( \left[ a,b\right] ,X\right) $, 
$n=\left \lceil \nu \right \rceil $, $\nu >0$. Then $D_{\ast a}^{\nu
}f\left( x\right) $ is continuous in $x\in \left[ a,b\right] $.
\end{proposition}

\begin{proposition}
\label{p23.}(\cite{11}) Let $f\in C^{m}\left( \left[ a,b\right] ,X\right) $, 
$m=\left \lceil \alpha \right \rceil $, $\alpha >0$. Then $D_{b-}^{\nu
}f\left( x\right) $ is continuous in $x\in \left[ a,b\right] $.
\end{proposition}

We also mention

\begin{proposition}
\label{p24.}(\cite{11}) Let $f\in C^{m-1}\left( \left[ a,b\right] ,X\right) $%
, $f^{\left( m\right) }\in L_{\infty }\left( \left[ a,b\right] ,X\right) $, $%
m=\left \lceil \alpha \right \rceil $, $\alpha >0$ and 
\begin{equation}
D_{\ast x_{0}}^{\alpha }f\left( x\right) =\frac{1}{\Gamma \left( m-\alpha
\right) }\int_{x_{0}}^{x}\left( x-t\right) ^{m-\alpha -1}f^{\left( m\right)
}\left( t\right) dt,  \tag{84}  \label{84}
\end{equation}%
for all $x,x_{0}\in \left[ a,b\right] :x\geq x_{0}.$

Then $D_{\ast x_{0}}^{\alpha }f\left( x\right) $ is continuous in $x_{0}$.
\end{proposition}

\begin{proposition}
\label{p25.}(\cite{11}) Let $f\in C^{m-1}\left( \left[ a,b\right] ,X\right) $%
, $f^{\left( m\right) }\in L_{\infty }\left( \left[ a,b\right] ,X\right) $, $%
m=\left \lceil \alpha \right \rceil $, $\alpha >0$ and 
\begin{equation}
D_{x_{0}-}^{\alpha }f\left( x\right) =\frac{\left( -1\right) ^{m}}{\Gamma
\left( m-\alpha \right) }\int_{x}^{x_{0}}\left( \zeta -x\right) ^{m-\alpha
-1}f^{\left( m\right) }\left( \zeta \right) d\zeta ,  \tag{85}  \label{85}
\end{equation}%
for all $x,x_{0}\in \left[ a,b\right] :x_{0}\geq x.$

Then $D_{x_{0}-}^{\alpha }f\left( x\right) $ is continuous in $x_{0}$.
\end{proposition}

\begin{corollary}
\label{c26}(\cite{11}) Let $f\in C^{m}\left( \left[ a,b\right] ,X\right) $, $%
m=\left \lceil \alpha \right \rceil $, $\alpha >0$, $x,x_{0}\in \left[ a,b%
\right] $. Then $D_{\ast x_{0}}^{a}f\left( x\right) ,$ $D_{x_{0}-}^{a}f%
\left( x\right) $ are jointly continuous functions in $\left( x,x_{0}\right) 
$ from $\left[ a,b\right] ^{2}$ into $X$, $X$ is a Banach space.
\end{corollary}

We need

\begin{theorem}
\label{t27.}(\cite{11}) Let $f:\left[ a,b\right] ^{2}\rightarrow X$ be
jointly continuous, $X$ is a Banach space. Consider 
\begin{equation}
G\left( x\right) =\omega _{1}\left( f\left( \cdot ,x\right) ,\delta ,\left[
x,b\right] \right) ,  \tag{86}  \label{86}
\end{equation}%
$\delta >0$, $x\in \left[ a,b\right] .$

Then $G$ is continuous on $\left[ a,b\right] .$
\end{theorem}

\begin{theorem}
\label{t28.}(\cite{11}) Let $f:\left[ a,b\right] ^{2}\rightarrow X$ be
jointly continuous, $X$ is a Banach space. Then 
\begin{equation}
H\left( x\right) =\omega _{1}\left( f\left( \cdot ,x\right) ,\delta ,\left[
a,x\right] \right) ,  \tag{87}  \label{87}
\end{equation}%
$x\in \left[ a,b\right] $, is continuous in $x\in \left[ a,b\right] $, $%
\delta >0$.
\end{theorem}

We make

\begin{remark}
\label{r29}(\cite{11}) Let $f\in C^{n-1}\left( \left[ a,b\right] \right) $, $%
f^{\left( n\right) }\in L_{\infty }\left( \left[ a,b\right] \right) $, $%
n=\left \lceil \nu \right \rceil $, $\nu >0$, $\nu \notin \mathbb{N}$. Then 
\begin{equation}
\left \Vert D_{\ast a}^{\nu }f\left( x\right) \right \Vert \leq \frac{\left
\Vert f^{\left( n\right) }\right \Vert _{L_{\infty }\left( \left[ a,b\right]
,X\right) }}{\Gamma \left( n-\nu +1\right) }\left( x-a\right) ^{n-\nu }\text{%
, \ }\forall \text{ }x\in \left[ a,b\right] .  \tag{88}  \label{88}
\end{equation}%
Thus we observe ($\delta >0$) 
\begin{equation}
\omega _{1}\left( D_{\ast a}^{\nu }f,\delta \right) =\underset{\left \vert
x-y\right \vert \leq \delta }{\underset{x,y\in \left[ a,b\right] }{\sup }}%
\left \Vert D_{\ast a}^{\nu }f\left( x\right) -D_{\ast a}^{\nu }f\left(
y\right) \right \Vert \leq  \tag{89}  \label{89}
\end{equation}%
\begin{equation*}
\underset{\left \vert x-y\right \vert \leq \delta }{\underset{x,y\in \left[
a,b\right] }{\sup }}\left( \frac{\left \Vert f^{\left( n\right) }\right
\Vert _{L_{\infty }\left( \left[ a,b\right] ,X\right) }}{\Gamma \left( n-\nu
+1\right) }\left( x-a\right) ^{n-\nu }+\frac{\left \Vert f^{\left( n\right)
}\right \Vert _{L_{\infty }\left( \left[ a,b\right] ,X\right) }}{\Gamma
\left( n-\nu +1\right) }\left( y-a\right) ^{n-\nu }\right)
\end{equation*}%
\begin{equation*}
\leq \frac{2\left \Vert f^{\left( n\right) }\right \Vert _{L_{\infty }\left( %
\left[ a,b\right] ,X\right) }}{\Gamma \left( n-\nu +1\right) }\left(
b-a\right) ^{n-\nu }.
\end{equation*}%
Consequently 
\begin{equation}
\omega _{1}\left( D_{\ast a}^{\nu }f,\delta \right) \leq \frac{2\left \Vert
f^{\left( n\right) }\right \Vert _{L_{\infty }\left( \left[ a,b\right]
,X\right) }}{\Gamma \left( n-\nu +1\right) }\left( b-a\right) ^{n-\nu }. 
\tag{90}  \label{90}
\end{equation}%
Similarly, let $f\in C^{m-1}\left( \left[ a,b\right] \right) $, $f^{\left(
m\right) }\in L_{\infty }\left( \left[ a,b\right] \right) $, $m=\left \lceil
\alpha \right \rceil $, $\alpha >0$, $\alpha \notin \mathbb{N}$, then 
\begin{equation}
\omega _{1}\left( D_{b-}^{\alpha }f,\delta \right) \leq \frac{2\left \Vert
f^{\left( m\right) }\right \Vert _{L_{\infty }\left( \left[ a,b\right]
,X\right) }}{\Gamma \left( m-\alpha +1\right) }\left( b-a\right) ^{m-\alpha
}.  \tag{91}  \label{91}
\end{equation}%
So for $f\in C^{m-1}\left( \left[ a,b\right] \right) $, $f^{\left( m\right)
}\in L_{\infty }\left( \left[ a,b\right] \right) $, $m=\left \lceil \alpha
\right \rceil $, $\alpha >0$, $\alpha \notin \mathbb{N}$, we find 
\begin{equation}
\underset{x_{0}\in \left[ a,b\right] }{\sup }\omega _{1}\left( D_{\ast
x_{0}}^{\alpha }f,\delta \right) _{\left[ x_{0},b\right] }\leq \frac{2\left
\Vert f^{\left( m\right) }\right \Vert _{L_{\infty }\left( \left[ a,b\right]
,X\right) }}{\Gamma \left( m-\alpha +1\right) }\left( b-a\right) ^{m-\alpha
},  \tag{92}  \label{92}
\end{equation}%
and 
\begin{equation}
\underset{x_{0}\in \left[ a,b\right] }{\sup }\omega _{1}\left(
D_{x_{0}-}^{\alpha }f,\delta \right) _{\left[ a,x_{0}\right] }\leq \frac{%
2\left \Vert f^{\left( m\right) }\right \Vert _{L_{\infty }\left( \left[ a,b%
\right] ,X\right) }}{\Gamma \left( m-\alpha +1\right) }\left( b-a\right)
^{m-\alpha }.  \tag{93}  \label{93}
\end{equation}
\end{remark}

By \cite{12} we get that $D_{\ast x_{0}}^{\alpha }f\in C\left( \left[ x_{0},b%
\right] ,X\right) $, and by \cite{10} we obtain that $D_{x_{0}-}^{\alpha
}f\in C\left( \left[ a,x_{0}\right] ,X\right) .$

We present the following $X$-valued fractional approximation result by
neural networks.

\begin{theorem}
\label{t30.}Let $\alpha >0$, $N=\left \lceil \alpha \right \rceil $, $\alpha
\notin \mathbb{N}$, $f\in C^{N}\left( \left[ a,b\right] ,X\right) $, $%
0<\beta <1$, $m\in \mathbb{N},$ $x\in \left[ a,b\right] $, $n\in \mathbb{N}%
:n^{1-\beta }>2.$ Then

i) 
\begin{equation*}
\left \Vert A_{n}\left( f,x\right) -\sum_{j=1}^{N-1}\frac{f^{\left( j\right)
}\left( x\right) }{j!}A_{n}\left( \left( \cdot -x\right) ^{j}\right) \left(
x\right) -f\left( x\right) \right \Vert \leq
\end{equation*}%
\begin{equation*}
\frac{2\left( \sqrt[2m]{1+4^{m}}\right) }{\Gamma \left( \alpha +1\right) }%
\left \{ \frac{\left( \omega _{1}\left( D_{x-}^{\alpha }f,\frac{1}{n^{\beta }%
}\right) _{\left[ a,x\right] }+\omega _{1}\left( D_{\ast x}^{\alpha }f,\frac{%
1}{n^{\beta }}\right) _{\left[ x,b\right] }\right) }{n^{\alpha \beta }}%
+\right.
\end{equation*}%
\begin{equation}
\left. \frac{1}{4m\left( n^{1-\beta }-2\right) ^{2m}}\left( \left \Vert
D_{x-}^{\alpha }f\right \Vert _{\infty ,\left[ a,x\right] }\left( x-a\right)
^{\alpha }+\left \Vert D_{\ast x}^{\alpha }f\right \Vert _{\infty ,\left[ x,b%
\right] }\left( b-x\right) ^{\alpha }\right) \right \} ,  \tag{94}
\label{94}
\end{equation}

ii) if $f^{\left( j\right) }\left( x\right) =0$, for $j=1,...,N-1$, we have 
\begin{equation*}
\left \Vert A_{n}\left( f,x\right) -f\left( x\right) \right \Vert \leq \frac{%
2\left( \sqrt[2m]{1+4^{m}}\right) }{\Gamma \left( \alpha +1\right) }
\end{equation*}%
\begin{equation*}
\left \{ \frac{\left( \omega _{1}\left( D_{x-}^{\alpha }f,\frac{1}{n^{\beta }%
}\right) _{\left[ a,x\right] }+\omega _{1}\left( D_{\ast x}^{\alpha }f,\frac{%
1}{n^{\beta }}\right) _{\left[ x,b\right] }\right) }{n^{\alpha \beta }}%
+\right.
\end{equation*}%
\begin{equation}
\left. \frac{1}{4m\left( n^{1-\beta }-2\right) ^{2m}}\left( \left \Vert
D_{x-}^{\alpha }f\right \Vert _{\infty ,\left[ a,x\right] }\left( x-a\right)
^{\alpha }+\left \Vert D_{\ast x}^{\alpha }f\right \Vert _{\infty ,\left[ x,b%
\right] }\left( b-x\right) ^{\alpha }\right) \right \} ,  \tag{95}
\label{95}
\end{equation}

iii) 
\begin{equation*}
\left \Vert A_{n}\left( f,x\right) -f\left( x\right) \right \Vert \leq
2\left( \sqrt[2m]{1+4^{m}}\right) \cdot
\end{equation*}%
\begin{equation*}
\left \{ \sum_{j=1}^{N-1}\frac{\left \Vert f^{\left( j\right) }\left(
x\right) \right \Vert }{j!}\left \{ \frac{1}{n^{\beta j}}+\frac{\left(
b-a\right) ^{j}}{4m\left( n^{1-\beta }-2\right) ^{2m}}\right \} +\right.
\end{equation*}%
\begin{equation*}
\frac{1}{\Gamma \left( \alpha +1\right) }\left \{ \frac{\left( \omega
_{1}\left( D_{x-}^{\alpha }f,\frac{1}{n^{\beta }}\right) _{\left[ a,x\right]
}+\omega _{1}\left( D_{\ast x}^{\alpha }f,\frac{1}{n^{\beta }}\right) _{%
\left[ x,b\right] }\right) }{n^{\alpha \beta }}+\right.
\end{equation*}%
\begin{equation}
\left. \left. \frac{1}{4m\left( n^{1-\beta }-2\right) ^{2m}}\left( \left
\Vert D_{x-}^{\alpha }f\right \Vert _{\infty ,\left[ a,x\right] }\left(
x-a\right) ^{\alpha }+\left \Vert D_{\ast x}^{\alpha }f\right \Vert _{\infty
,\left[ x,b\right] }\left( b-x\right) ^{\alpha }\right) \right \} \right \} ,
\tag{96}  \label{96}
\end{equation}%
$\forall $ $x\in \left[ a,b\right] ,$

and

iv) 
\begin{equation*}
\left \Vert A_{n}f-f\right \Vert _{\infty }\leq 2\left( \sqrt[2m]{1+4^{m}}%
\right) \cdot
\end{equation*}%
\begin{equation*}
\left \{ \sum_{j=1}^{N-1}\frac{\left \Vert f^{\left( j\right) }\right \Vert
_{\infty }}{j!}\left \{ \frac{1}{n^{\beta j}}+\frac{\left( b-a\right) ^{j}}{%
4m\left( n^{1-\beta }-2\right) ^{2m}}\right \} +\right.
\end{equation*}%
\begin{equation*}
\frac{1}{\Gamma \left( \alpha +1\right) }\left \{ \frac{\left( \underset{%
x\in \left[ a,b\right] }{\sup }\omega _{1}\left( D_{x-}^{\alpha }f,\frac{1}{%
n^{\beta }}\right) _{\left[ a,x\right] }+\underset{x\in \left[ a,b\right] }{%
\sup }\omega _{1}\left( D_{\ast x}^{\alpha }f,\frac{1}{n^{\beta }}\right) _{%
\left[ x,b\right] }\right) }{n^{\alpha \beta }}+\right.
\end{equation*}%
\begin{equation}
\left. \left. \frac{\left( b-a\right) ^{\alpha }}{4m\left( n^{1-\beta
}-2\right) ^{2m}}\left( \underset{x\in \left[ a,b\right] }{\sup }\left \Vert
D_{x-}^{\alpha }f\right \Vert _{\infty ,\left[ a,x\right] }+\underset{x\in %
\left[ a,b\right] }{\sup }\left \Vert D_{\ast x}^{\alpha }f\right \Vert
_{\infty ,\left[ x,b\right] }\right) \right \} \right \} .  \tag{97}
\label{97}
\end{equation}%
Above, when $N=1$ the sum $\sum_{j=1}^{N-1}\cdot =0.$

As we see here we obtain $X$-valued fractionally type pointwise and uniform
convergence with rates of $A_{n}\rightarrow I$ the unit operator, as $%
n\rightarrow \infty .$
\end{theorem}

\begin{proof}
Let $x\in \left[ a,b\right] $. We have that $D_{x-}^{\alpha }f\left(
x\right) =D_{\ast x}^{\alpha }f\left( x\right) =0.$

From Theorem \ref{t19.}, we get by the left Caputo fractional Taylor formula
that 
\begin{equation}
f\left( \frac{k}{n}\right) =\sum_{j=0}^{N-1}\frac{f^{\left( j\right) }\left(
x\right) }{j!}\left( \frac{k}{n}-x\right) ^{j}+  \tag{98}  \label{98}
\end{equation}%
\begin{equation*}
\frac{1}{\Gamma \left( \alpha \right) }\int_{x}^{\frac{k}{n}}\left( \frac{k}{%
n}-J\right) ^{\alpha -1}\left( D_{\ast x}^{\alpha }f\left( J\right) -D_{\ast
x}^{\alpha }f\left( x\right) \right) dJ,
\end{equation*}%
for all $x\leq \frac{k}{n}\leq b.$

Also from Theorem \ref{t20}, using the right Caputo fractional Taylor
formula we get%
\begin{equation}
f\left( \frac{k}{n}\right) =\sum_{j=0}^{N-1}\frac{f^{\left( j\right) }\left(
x\right) }{j!}\left( \frac{k}{n}-x\right) ^{j}+  \tag{99}  \label{99}
\end{equation}%
\begin{equation*}
\frac{1}{\Gamma \left( \alpha \right) }\int_{\frac{k}{n}}^{x}\left( J-\frac{k%
}{n}\right) ^{\alpha -1}\left( D_{x-}^{\alpha }f\left( J\right)
-D_{x-}^{\alpha }f\left( x\right) \right) dJ,
\end{equation*}%
for all $a\leq \frac{k}{n}\leq x.$

Hence we have 
\begin{equation}
f\left( \frac{k}{n}\right) \Phi \left( nx-k\right) =\sum_{j=0}^{N-1}\frac{%
f^{\left( j\right) }\left( x\right) }{j!}\Phi \left( nx-k\right) \left( 
\frac{k}{n}-x\right) ^{j}+  \tag{100}  \label{100}
\end{equation}%
\begin{equation*}
\frac{\Phi \left( nx-k\right) }{\Gamma \left( \alpha \right) }\int_{x}^{%
\frac{k}{n}}\left( \frac{k}{n}-J\right) ^{\alpha -1}\left( D_{\ast
x}^{\alpha }f\left( J\right) -D_{\ast x}^{\alpha }f\left( x\right) \right)
dJ,
\end{equation*}%
for all $x\leq \frac{k}{n}\leq b$, iff $\left\lceil nx\right\rceil \leq
k\leq \left\lfloor nb\right\rfloor $, and 
\begin{equation}
f\left( \frac{k}{n}\right) \Phi \left( nx-k\right) =\sum_{j=0}^{N-1}\frac{%
f^{\left( j\right) }\left( x\right) }{j!}\Phi \left( nx-k\right) \left( 
\frac{k}{n}-x\right) ^{j}+  \tag{101}  \label{101}
\end{equation}%
\begin{equation*}
\frac{\Phi \left( nx-k\right) }{\Gamma \left( \alpha \right) }\int_{\frac{k}{%
n}}^{x}\left( J-\frac{k}{n}\right) ^{\alpha -1}\left( D_{x-}^{\alpha
}f\left( J\right) -D_{x-}^{\alpha }f\left( x\right) \right) dJ,
\end{equation*}%
for all $a\leq \frac{k}{n}\leq x$, iff $\left\lceil na\right\rceil \leq
k\leq \left\lfloor nx\right\rfloor .$

Therefore it holds 
\begin{equation}
\sum_{k=\left \lfloor nx\right \rfloor +1}^{\left \lfloor nb\right \rfloor
}f\left( \frac{k}{n}\right) \Phi \left( nx-k\right) =\sum_{j=0}^{N-1}\frac{%
f^{\left( j\right) }\left( x\right) }{j!}\sum_{k=\left \lfloor nx\right
\rfloor +1}^{\left \lfloor nb\right \rfloor }\Phi \left( nx-k\right) \left( 
\frac{k}{n}-x\right) ^{j}+  \tag{102}  \label{102}
\end{equation}%
\begin{equation*}
\frac{1}{\Gamma \left( \alpha \right) }\sum_{k=\left \lfloor nx\right
\rfloor +1}^{\left \lfloor nb\right \rfloor }\Phi \left( nx-k\right)
\int_{x}^{\frac{k}{n}}\left( \frac{k}{n}-J\right) ^{\alpha -1}\left( D_{\ast
x}^{\alpha }f\left( J\right) -D_{\ast x}^{\alpha }f\left( x\right) \right)
dJ,
\end{equation*}%
and 
\begin{equation}
\sum_{k=\left \lceil na\right \rceil }^{\left \lfloor nx\right \rfloor
}f\left( \frac{k}{n}\right) \Phi \left( nx-k\right) =\sum_{j=0}^{N-1}\frac{%
f^{\left( j\right) }\left( x\right) }{j!}\sum_{k=\left \lceil na\right
\rceil }^{\left \lfloor nx\right \rfloor }\Phi \left( nx-k\right) \left( 
\frac{k}{n}-x\right) ^{j}+  \tag{103}  \label{103}
\end{equation}%
\begin{equation*}
\frac{1}{\Gamma \left( \alpha \right) }\sum_{k=\left \lceil na\right \rceil
}^{\left \lfloor nx\right \rfloor }\Phi \left( nx-k\right) \int_{\frac{k}{n}%
}^{x}\left( J-\frac{k}{n}\right) ^{\alpha -1}\left( D_{x-}^{\alpha }f\left(
J\right) -D_{x-}^{\alpha }f\left( x\right) \right) dJ.
\end{equation*}%
Adding the last two equalities (\ref{102}) and (\ref{103}) obtain 
\begin{equation}
A_{n}^{\ast }\left( f,x\right) =\sum_{k=\left \lceil na\right \rceil
}^{\left \lfloor nb\right \rfloor }f\left( \frac{k}{n}\right) \Phi \left(
nx-k\right) =  \tag{104}  \label{104}
\end{equation}%
\begin{equation*}
\sum_{j=0}^{N-1}\frac{f^{\left( j\right) }\left( x\right) }{j!}\sum_{k=\left
\lceil na\right \rceil }^{\left \lfloor nb\right \rfloor }\Phi \left(
nx-k\right) \left( \frac{k}{n}-x\right) ^{j}+
\end{equation*}%
\begin{equation*}
\frac{1}{\Gamma \left( \alpha \right) }\left \{ \sum_{k=\left \lceil
na\right \rceil }^{\left \lfloor nx\right \rfloor }\Phi \left( nx-k\right)
\int_{\frac{k}{n}}^{x}\left( J-\frac{k}{n}\right) ^{\alpha -1}\left(
D_{x-}^{\alpha }f\left( J\right) -D_{x-}^{\alpha }f\left( x\right) \right)
dJ+\right.
\end{equation*}%
\begin{equation*}
\left. \sum_{k=\left \lfloor nx\right \rfloor +1}^{\left \lfloor nb\right
\rfloor }\Phi \left( nx-k\right) \int_{x}^{\frac{k}{n}}\left( \frac{k}{n}%
-J\right) ^{\alpha -1}\left( D_{\ast x}^{\alpha }f\left( J\right) -D_{\ast
x}^{\alpha }f\left( x\right) \right) dJ\right \} .
\end{equation*}%
So we have derived 
\begin{equation}
A_{n}^{\ast }\left( f,x\right) -f\left( x\right) \left( \sum_{k=\left \lceil
na\right \rceil }^{\left \lfloor nb\right \rfloor }\Phi \left( nx-k\right)
\right) =  \tag{105}  \label{105}
\end{equation}%
\begin{equation*}
\sum_{j=1}^{N-1}\frac{f^{\left( j\right) }\left( x\right) }{j!}A_{n}^{\ast
}\left( \left( \cdot -x\right) ^{j}\right) +u_{n}\left( x\right) ,
\end{equation*}%
where 
\begin{equation*}
u_{n}\left( x\right) :=\frac{1}{\Gamma \left( \alpha \right) }\left \{
\sum_{k=\left \lceil na\right \rceil }^{\left \lfloor nx\right \rfloor }\Phi
\left( nx-k\right) \int_{\frac{k}{n}}^{x}\left( J-\frac{k}{n}\right)
^{\alpha -1}\left( D_{x-}^{\alpha }f\left( J\right) -D_{x-}^{\alpha }f\left(
x\right) \right) dJ\right.
\end{equation*}%
\begin{equation}
\left. +\sum_{k=\left \lfloor nx\right \rfloor +1}^{\left \lfloor nb\right
\rfloor }\Phi \left( nx-k\right) \int_{x}^{\frac{k}{n}}\left( \frac{k}{n}%
-J\right) ^{\alpha -1}\left( D_{\ast x}^{\alpha }f\left( J\right) -D_{\ast
x}^{\alpha }f\left( x\right) \right) dJ\right \} .  \tag{106}  \label{106}
\end{equation}%
We set 
\begin{equation}
u_{1n}\left( x\right) :=\frac{1}{\Gamma \left( \alpha \right) }\sum_{k=\left
\lceil na\right \rceil }^{\left \lfloor nx\right \rfloor }\Phi \left(
nx-k\right) \int_{\frac{k}{n}}^{x}\left( J-\frac{k}{n}\right) ^{\alpha
-1}\left( D_{x-}^{\alpha }f\left( J\right) -D_{x-}^{\alpha }f\left( x\right)
\right) dJ,  \tag{107}  \label{107}
\end{equation}%
and 
\begin{equation}
u_{2n}:=\frac{1}{\Gamma \left( \alpha \right) }\sum_{k=\left \lceil nx\right
\rceil +1}^{\left \lfloor nb\right \rfloor }\Phi \left( nx-k\right)
\int_{x}^{\frac{k}{n}}\left( \frac{k}{n}-J\right) ^{\alpha -1}\left( D_{\ast
x}^{\alpha }f\left( J\right) -D_{\ast x}^{\alpha }f\left( x\right) \right)
dJ,  \tag{108}  \label{108}
\end{equation}%
i.e. 
\begin{equation}
u_{n}\left( x\right) =u_{1n}\left( x\right) +u_{2n}\left( x\right) . 
\tag{109}  \label{109}
\end{equation}%
We assume $b-a>\frac{1}{n^{\beta }}$, $0<\beta <1$, which is always the case
for large enough $n\in \mathbb{N}$, that is when $n>\left \lceil \left(
b-a\right) ^{-\frac{1}{\beta }}\right \rceil $. It is always true that
either $\left \vert \frac{k}{n}-x\right \vert \leq \frac{1}{n^{\beta }}$ or $%
\left \vert \frac{k}{n}-x\right \vert >\frac{1}{n^{\beta }}.$

For $k=\left \lceil na\right \rceil ,...,\left \lfloor nx\right \rfloor $,
we consider 
\begin{equation}
\gamma _{1k}:=\left \Vert \int_{\frac{k}{n}}^{x}\left( J-\frac{k}{n}\right)
^{\alpha -1}\left( D_{x-}^{\alpha }f\left( J\right) -D_{x-}^{\alpha }f\left(
x\right) \right) dJ\right \Vert =  \tag{110}  \label{110}
\end{equation}%
\begin{equation*}
\left \Vert \int_{\frac{k}{n}}^{x}\left( J-\frac{k}{n}\right) ^{\alpha
-1}D_{x-}^{\alpha }f\left( J\right) dJ\right \Vert \leq \int_{\frac{k}{n}%
}^{x}\left( J-\frac{k}{n}\right) ^{\alpha -1}\left \Vert D_{x-}^{\alpha
}f\left( J\right) \right \Vert dJ\leq
\end{equation*}%
\begin{equation}
\left \Vert D_{x-}^{\alpha }f\left( J\right) \right \Vert _{\infty ,\left[
a,x\right] }\frac{\left( x-\frac{k}{n}\right) ^{\alpha }}{\alpha }\leq \left
\Vert D_{x-}^{\alpha }f\right \Vert _{\infty ,\left[ a,x\right] }\frac{%
\left( x-a\right) ^{\alpha }}{\alpha }.  \tag{111}  \label{111}
\end{equation}%
That is 
\begin{equation}
\gamma _{1k}\leq \left \Vert D_{x-}^{\alpha }f\right \Vert _{\infty ,\left[
a,x\right] }\frac{\left( x-a\right) ^{\alpha }}{\alpha },  \tag{112}
\label{112}
\end{equation}%
for $k=\left \lceil na\right \rceil ,...,\left \lfloor nx\right \rfloor .$

Also we have in case of $\left \vert \frac{k}{n}-x\right \vert \leq \frac{1}{%
n^{\beta }}$ that 
\begin{equation*}
\gamma _{1k}\leq \int_{\frac{k}{n}}^{x}\left( J-\frac{k}{n}\right) ^{\alpha
-1}\left \Vert D_{x-}^{\alpha }f\left( J\right) -D_{x-}^{\alpha }f\left(
x\right) \right \Vert dJ\leq
\end{equation*}%
\begin{equation}
\int_{\frac{k}{n}}^{x}\left( J-\frac{k}{n}\right) ^{\alpha -1}\omega
_{1}\left( D_{x-}^{\alpha }f,\left \vert J-x\right \vert \right) _{\left[ a,x%
\right] }dJ\leq  \tag{113}  \label{113}
\end{equation}%
\begin{equation*}
\omega _{1}\left( D_{x-}^{\alpha }f,\left \vert x-\frac{k}{n}\right \vert
\right) _{\left[ a,x\right] }\int_{\frac{k}{n}}^{x}\left( J-\frac{k}{n}%
\right) ^{\alpha -1}dJ\leq
\end{equation*}%
\begin{equation*}
\omega _{1}\left( D_{x-}^{\alpha }f,\frac{1}{n^{\beta }}\right) _{\left[ a,x%
\right] }\frac{\left( x-\frac{k}{n}\right) ^{\alpha }}{\alpha }\leq \omega
_{1}\left( D_{x-}^{\alpha }f,\frac{1}{n^{\beta }}\right) _{\left[ a,x\right]
}\frac{1}{\alpha n^{\alpha \beta }}.
\end{equation*}%
That is when $\left \vert \frac{k}{n}-x\right \vert \leq \frac{1}{n^{\beta }}
$, then 
\begin{equation}
\gamma _{1k}\leq \frac{\omega _{1}\left( D_{x-}^{\alpha }f,\frac{1}{n^{\beta
}}\right) _{\left[ a,x\right] }}{\alpha n^{\alpha \beta }}.  \tag{114}
\label{114}
\end{equation}%
Consequently we obtain 
\begin{equation}
\left \Vert u_{1n}\left( x\right) \right \Vert \leq \frac{1}{\Gamma \left(
\alpha \right) }\sum_{k=\left \lceil na\right \rceil }^{\left \lfloor
nx\right \rfloor }\Phi \left( nx-k\right) \gamma _{1k}=  \tag{115}
\label{115}
\end{equation}%
\begin{equation*}
\frac{1}{\Gamma \left( \alpha \right) }\left \{ \sum_{\left \{ 
\begin{array}{l}
k=\left \lceil na\right \rceil \\ 
:\left \vert \frac{k}{n}-x\right \vert \leq \frac{1}{n^{\beta }}%
\end{array}%
\right. }^{\left \lfloor nx\right \rfloor }\Phi \left( nx-k\right) \gamma
_{1k}+\sum_{\left \{ 
\begin{array}{l}
k=\left \lceil na\right \rceil \\ 
:\left \vert \frac{k}{n}-x\right \vert >\frac{1}{n^{\beta }}%
\end{array}%
\right. }^{\left \lfloor nx\right \rfloor }\Phi \left( nx-k\right) \gamma
_{1k}\right \} \leq
\end{equation*}%
\begin{equation*}
\frac{1}{\Gamma \left( \alpha \right) }\left \{ \left( \sum_{\left \{ 
\begin{array}{l}
k=\left \lceil na\right \rceil \\ 
:\left \vert \frac{k}{n}-x\right \vert \leq \frac{1}{n^{\beta }}%
\end{array}%
\right. }^{\left \lfloor nx\right \rfloor }\Phi \left( nx-k\right) \right) 
\frac{\omega _{1}\left( D_{x-}^{\alpha }f,\frac{1}{n^{\beta }}\right) _{%
\left[ a,x\right] }}{\alpha n^{\alpha \beta }}+\right.
\end{equation*}%
\begin{equation}
\left. \left( \sum_{\left \{ 
\begin{array}{l}
k=\left \lceil na\right \rceil \\ 
:\left \vert \frac{k}{n}-x\right \vert >\frac{1}{n^{\beta }}%
\end{array}%
\right. }^{\left \lfloor nx\right \rfloor }\Phi \left( nx-k\right) \right)
\left \Vert D_{x-}^{\alpha }f\right \Vert _{\infty ,\left[ a,x\right] }\frac{%
\left( x-a\right) ^{\alpha }}{\alpha }\right \} \leq  \tag{116}  \label{116}
\end{equation}%
\begin{equation*}
\frac{1}{\Gamma \left( \alpha +1\right) }\left \{ \frac{\omega _{1}\left(
D_{x-}^{\alpha }f,\frac{1}{n^{\beta }}\right) _{\left[ a,x\right] }}{%
n^{\alpha \beta }}+\right.
\end{equation*}%
\begin{equation*}
\left. \left( \sum_{\left \{ 
\begin{array}{l}
k=-\infty \\ 
:\left \vert nx-k\right \vert >n^{1-\beta }%
\end{array}%
\right. }^{\infty }\Phi \left( nx-k\right) \right) \left \Vert
D_{x-}^{\alpha }f\right \Vert _{\infty ,\left[ a,x\right] }\left( x-a\right)
^{\alpha }\right \} \leq
\end{equation*}%
\begin{equation*}
\frac{1}{\Gamma \left( \alpha +1\right) }\left \{ \frac{\omega _{1}\left(
D_{x-}^{\alpha }f,\frac{1}{n^{\beta }}\right) _{\left[ a,x\right] }}{%
n^{\alpha \beta }}+\frac{\left \Vert D_{x-}^{\alpha }f\right \Vert _{\infty ,%
\left[ a,x\right] }\left( x-a\right) ^{\alpha }}{4m\left( n^{1-\beta
}-2\right) ^{2m}}\right \} .
\end{equation*}%
So we have proved that 
\begin{equation}
\left \Vert u_{1n}\left( x\right) \right \Vert \leq \frac{1}{\Gamma \left(
\alpha +1\right) }\left \{ \frac{\omega _{1}\left( D_{x-}^{\alpha }f,\frac{1%
}{n^{\beta }}\right) _{\left[ a,x\right] }}{n^{\alpha \beta }}+\frac{\left
\Vert D_{x-}^{\alpha }f\right \Vert _{\infty ,\left[ a,x\right] }\left(
x-a\right) ^{\alpha }}{4m\left( n^{1-\beta }-2\right) ^{2m}}\right \} . 
\tag{117}  \label{117}
\end{equation}

Next when $k=\left \lfloor nx\right \rfloor +1,...,\left \lfloor
nb\right
\rfloor $ we consider 
\begin{equation}
\gamma _{2k}:=\left \Vert \int_{x}^{\frac{k}{n}}\left( \frac{k}{n}-J\right)
^{\alpha -1}\left( D_{\ast x}^{\alpha }f\left( J\right) -D_{\ast x}^{\alpha
}f\left( x\right) \right) dJ\right \Vert \leq  \tag{118}  \label{118}
\end{equation}%
\begin{equation*}
\int_{x}^{\frac{k}{n}}\left( \frac{k}{n}-J\right) ^{\alpha -1}\left \Vert
D_{\ast x}^{\alpha }f\left( J\right) -D_{\ast x}^{\alpha }f\left( x\right)
\right \Vert dJ=
\end{equation*}%
\begin{equation*}
\int_{x}^{\frac{k}{n}}\left( \frac{k}{n}-J\right) ^{\alpha -1}\left \Vert
D_{\ast x}^{\alpha }f\left( J\right) \right \Vert dJ\leq
\end{equation*}%
\begin{equation}
\left \Vert D_{\ast x}^{\alpha }f\right \Vert _{\infty ,\left[ x,b\right] }%
\frac{\left( \frac{k}{n}-x\right) ^{\alpha }}{\alpha }\leq \left \Vert
D_{\ast x}^{\alpha }f\right \Vert _{\infty ,\left[ x,b\right] }\frac{\left(
b-x\right) ^{\alpha }}{\alpha }.  \tag{119}  \label{119}
\end{equation}%
Therefore when $k=\left \lfloor nx\right \rfloor +1,...,\left \lfloor
nb\right \rfloor $ we get that

That is 
\begin{equation}
\gamma _{2k}\leq \left \Vert D_{\ast x}^{\alpha }f\right \Vert _{\infty , 
\left[ x,b\right] }\frac{\left( b-x\right) ^{\alpha }}{\alpha }.  \tag{120}
\label{120}
\end{equation}%
In case of $\left \vert \frac{k}{n}-x\right \vert \leq \frac{1}{n^{\beta }}$
we have 
\begin{equation*}
\gamma _{2k}\leq \int_{x}^{\frac{k}{n}}\left( \frac{k}{n}-J\right) ^{\alpha
-1}\omega _{1}\left( D_{\ast x}^{\alpha }f,\left \vert J-x\right \vert
\right) _{\left[ x,b\right] }dJ\leq
\end{equation*}%
\begin{equation}
\omega _{1}\left( D_{\ast x}^{\alpha }f,\left \vert \frac{k}{n}-x\right
\vert \right) _{\left[ x,b\right] }\int_{x}^{\frac{k}{n}}\left( \frac{k}{n}%
-J\right) ^{\alpha -1}dJ\leq  \tag{121}  \label{121}
\end{equation}%
\begin{equation*}
\omega _{1}\left( D_{\ast x}^{\alpha }f,\frac{1}{n^{\beta }}\right) _{\left[
x,b\right] }\frac{\left( \frac{k}{n}-x\right) ^{\alpha }}{\alpha }\leq
\omega _{1}\left( D_{\ast x}^{\alpha }f,\frac{1}{n^{\beta }}\right) _{\left[
x,b\right] }\frac{1}{\alpha n^{\alpha \beta }}.
\end{equation*}%
So when $\left \vert \frac{k}{n}-x\right \vert \leq \frac{1}{n^{\beta }}$ we
derived that 
\begin{equation}
\gamma _{2k}\leq \frac{\omega _{1}\left( D_{\ast x}^{\alpha }f,\frac{1}{%
n^{\beta }}\right) _{\left[ x,b\right] }}{\alpha n^{\alpha \beta }}. 
\tag{122}  \label{122}
\end{equation}%
Similarly we have that%
\begin{equation*}
\left \Vert u_{2n}\left( x\right) \right \Vert \leq \frac{1}{\Gamma \left(
\alpha \right) }\left( \sum_{k=\left \lfloor nx\right \rfloor +1}^{\left
\lfloor nb\right \rfloor }\Phi \left( nx-k\right) \gamma _{2k}\right) =
\end{equation*}%
\begin{equation}
\frac{1}{\Gamma \left( \alpha \right) }\left \{ \sum_{\left \{ 
\begin{array}{l}
k=\left \lfloor nx\right \rfloor +1 \\ 
:\left \vert \frac{k}{n}-x\right \vert \leq \frac{1}{n^{\beta }}%
\end{array}%
\right. }^{\left \lfloor nb\right \rfloor }\Phi \left( nx-k\right) \gamma
_{2k}+\sum_{\left \{ 
\begin{array}{l}
k=\left \lfloor nx\right \rfloor +1 \\ 
:\left \vert \frac{k}{n}-x\right \vert >\frac{1}{n^{\beta }}%
\end{array}%
\right. }^{\left \lfloor nb\right \rfloor }\Phi \left( nx-k\right) \gamma
_{2k}\right \} \leq  \tag{123}  \label{123}
\end{equation}%
\begin{equation*}
\frac{1}{\Gamma \left( \alpha \right) }\left \{ \left( \sum_{\left \{ 
\begin{array}{l}
k=\left \lfloor nx\right \rfloor +1 \\ 
:\left \vert \frac{k}{n}-x\right \vert \leq \frac{1}{n^{\beta }}%
\end{array}%
\right. }^{\left \lfloor nb\right \rfloor }\Phi \left( nx-k\right) \right) 
\frac{\omega _{1}\left( D_{\ast x}^{\alpha }f,\frac{1}{n^{\beta }}\right) _{%
\left[ x,b\right] }}{\alpha n^{\alpha \beta }}+\right.
\end{equation*}%
\begin{equation*}
\left. \left( \sum_{\left \{ 
\begin{array}{l}
k=\left \lfloor nx\right \rfloor +1 \\ 
:\left \vert \frac{k}{n}-x\right \vert >\frac{1}{n^{\beta }}%
\end{array}%
\right. }^{\left \lfloor nb\right \rfloor }\Phi \left( nx-k\right) \right)
\left \Vert D_{\ast x}^{\alpha }f\right \Vert _{\infty ,\left[ x,b\right] }%
\frac{\left( b-x\right) ^{\alpha }}{\alpha }\right \} \leq
\end{equation*}%
\begin{equation*}
\frac{1}{\Gamma \left( \alpha +1\right) }\left \{ \frac{\omega _{1}\left(
D_{\ast x}^{\alpha }f,\frac{1}{n^{\beta }}\right) _{\left[ x,b\right] }}{%
n^{\alpha \beta }}+\right.
\end{equation*}%
\begin{equation}
\left. \left( \sum_{\left \{ 
\begin{array}{l}
k=-\infty \\ 
:\left \vert \frac{k}{n}-x\right \vert >\frac{1}{n^{\beta }}%
\end{array}%
\right. }^{\infty }\Phi \left( nx-k\right) \right) \left \Vert D_{\ast
x}^{\alpha }f\right \Vert _{\infty ,\left[ x,b\right] }\left( b-x\right)
^{\alpha }\right \} \leq  \tag{124}  \label{124}
\end{equation}%
\begin{equation*}
\frac{1}{\Gamma \left( \alpha +1\right) }\left \{ \frac{\omega _{1}\left(
D_{\ast x}^{\alpha }f,\frac{1}{n^{\beta }}\right) _{\left[ x,b\right] }}{%
n^{\alpha \beta }}+\frac{\left \Vert D_{\ast x}^{\alpha }f\right \Vert
_{\infty ,\left[ x,b\right] }\left( b-x\right) ^{\alpha }}{4m\left(
n^{1-\beta }-2\right) ^{2m}}\right \} .
\end{equation*}%
So we have proved that 
\begin{equation}
\left \Vert u_{2n}\left( x\right) \right \Vert \leq \frac{1}{\Gamma \left(
\alpha +1\right) }\left \{ \frac{\omega _{1}\left( D_{\ast x}^{\alpha }f,%
\frac{1}{n^{\beta }}\right) _{\left[ x,b\right] }}{n^{\alpha \beta }}+\frac{%
\left \Vert D_{\ast x}^{\alpha }f\right \Vert _{\infty ,\left[ x,b\right]
}\left( b-x\right) ^{\alpha }}{4m\left( n^{1-\beta }-2\right) ^{2m}}\right
\} .  \tag{125}  \label{125}
\end{equation}%
Therefore 
\begin{equation*}
\left \Vert u_{n}\left( x\right) \right \Vert \leq \left \Vert u_{1n}\left(
x\right) \right \Vert +\left \Vert u_{2n}\left( x\right) \right \Vert \leq
\end{equation*}%
\begin{equation}
\frac{1}{\Gamma \left( \alpha +1\right) }\left \{ \frac{\omega _{1}\left(
D_{x-}^{\alpha }f,\frac{1}{n^{\beta }}\right) _{\left[ a,x\right] }+\omega
_{1}\left( D_{\ast x}^{\alpha }f,\frac{1}{n^{\beta }}\right) _{\left[ x,b%
\right] }}{n^{\alpha \beta }}+\right.  \tag{126}  \label{126}
\end{equation}%
\begin{equation*}
\left. \frac{1}{4m\left( n^{1-\beta }-2\right) ^{2m}}\left( \left \Vert
D_{x-}^{\alpha }f\right \Vert _{\infty ,\left[ a,x\right] }\left( x-a\right)
^{\alpha }+\left \Vert D_{\ast x}^{\alpha }f\right \Vert _{\infty ,\left[ x,b%
\right] }\left( b-x\right) ^{\alpha }\right) \right \} .
\end{equation*}%
From the proof of Theorem \ref{t14.} we get that 
\begin{equation}
\left \vert A_{n}^{\ast }\left( \left( \cdot -x\right) ^{j}\right) \left(
x\right) \right \vert \leq \frac{1}{n^{\beta j}}+\frac{\left( b-a\right) ^{j}%
}{4m\left( n^{1-\beta }-2\right) ^{2m}},  \tag{127}  \label{127}
\end{equation}%
for $j=1,...,N-1,$ $\forall $ $x\in \left[ a,b\right] .$

Putting things together, we have established 
\begin{equation}
\left\Vert A_{n}^{\ast }\left( f,x\right) -f\left( x\right) \left(
\sum_{k=\left\lceil na\right\rceil }^{\left\lfloor nb\right\rfloor }\Phi
\left( nx-k\right) \right) \right\Vert \leq \sum_{j=1}^{N-1}\frac{\left\Vert
f^{\left( j\right) }\left( x\right) \right\Vert }{j!}  \tag{128}  \label{128}
\end{equation}%
\begin{equation*}
\left[ \frac{1}{n^{\beta j}}+\frac{\left( b-a\right) ^{j}}{4m\left(
n^{1-\alpha }-2\right) ^{2m}}\right] +
\end{equation*}%
\begin{equation*}
\frac{1}{\Gamma \left( \alpha +1\right) }\left\{ \frac{\omega _{1}\left(
D_{x-}^{\alpha }f,\frac{1}{n^{\beta }}\right) _{\left[ a,x\right] }+\omega
_{1}\left( D_{\ast x}^{\alpha }f,\frac{1}{n^{\beta }}\right) _{\left[ x,b%
\right] }}{n^{\alpha \beta }}+\right.
\end{equation*}%
\begin{equation}
\left. \frac{1}{4m\left( n^{1-\beta }-2\right) ^{2m}}\left( \left\Vert
D_{x-}^{\alpha }f\right\Vert _{\infty ,\left[ a,x\right] }\left( x-a\right)
^{\alpha }+\left\Vert D_{\ast x}^{\alpha }f\right\Vert _{\infty ,\left[ x,b%
\right] }\left( b-x\right) ^{\alpha }\right) \right\} =:K_{n}\left( x\right)
.  \tag{129}  \label{129}
\end{equation}%
As a result we derive (see (\ref{41})) 
\begin{equation}
\left\Vert A_{n}\left( f,x\right) -f\left( x\right) \right\Vert \leq 2\left( 
\sqrt[2m]{1+4^{m}}\right) K_{n}\left( x\right) ,\text{ \ }\forall \text{ }%
x\in \left[ a,b\right] .  \tag{130}  \label{130}
\end{equation}%
We further have that 
\begin{equation}
\left\Vert K_{n}\right\Vert _{\infty }\leq \sum_{j=1}^{N-1}\frac{\left\Vert
f^{\left( j\right) }\right\Vert _{\infty }}{j!}\left[ \frac{1}{n^{\beta j}}%
+\left( b-a\right) ^{j}\frac{1}{4m\left( n^{1-\alpha }-2\right) ^{2m}}\right]
+  \tag{131}  \label{131}
\end{equation}%
\begin{equation*}
\frac{1}{\Gamma \left( \alpha +1\right) }\left\{ \frac{\left\{ \underset{%
x\in \left[ a,b\right] }{\sup }\left( \omega _{1}\left( D_{x-}^{\alpha }f,%
\frac{1}{n^{\beta }}\right) _{\left[ a,x\right] }\right) +\underset{x\in %
\left[ a,b\right] }{\sup }\left( \omega _{1}\left( D_{\ast x}^{\alpha }f,%
\frac{1}{n^{\beta }}\right) _{\left[ x,b\right] }\right) \right\} }{%
n^{\alpha \beta }}+\right.
\end{equation*}%
\begin{equation*}
\left( b-a\right) ^{\alpha }\frac{1}{4m\left( n^{1-\beta }-2\right) ^{2m}}
\end{equation*}%
\begin{equation*}
\left. \left\{ \left( \underset{x\in \left[ a,b\right] }{\sup }\left(
\left\Vert D_{x-}^{\alpha }f\right\Vert _{\infty ,\left[ a,x\right] }\right)
+\underset{x\in \left[ a,b\right] }{\sup }\left( \left\Vert D_{\ast
x}^{\alpha }f\right\Vert _{\infty ,\left[ x,b\right] }\right) \right)
\right\} \right\} =:E_{n}.
\end{equation*}%
Hence it holds 
\begin{equation}
\left\Vert A_{n}f-f\right\Vert _{\infty }\leq 2\left( \sqrt[2m]{1+4^{m}}%
\right) E_{n}.  \tag{132}  \label{132}
\end{equation}%
We observe the following:

We have 
\begin{equation}
\left( D_{x-}^{\alpha }f\right) \left( y\right) =\frac{\left( -1\right) ^{N}%
}{\Gamma \left( N-\alpha \right) }\int_{y}^{x}\left( J-y\right) ^{N-\alpha
-1}f^{\left( N\right) }\left( J\right) dJ,\text{ \ }\forall \text{ }y\in %
\left[ a,x\right]  \tag{133}  \label{133}
\end{equation}%
and 
\begin{equation*}
\left \Vert \left( D_{x-}^{\alpha }f\right) \left( y\right) \right \Vert
\leq \frac{1}{\Gamma \left( N-\alpha \right) }\left( \int_{y}^{x}\left(
J-y\right) ^{N-\alpha -1}dJ\right) \left \Vert f^{\left( N\right) }\right
\Vert _{\infty }=
\end{equation*}%
\begin{equation*}
\frac{1}{\Gamma \left( N-\alpha \right) }\frac{\left( x-y\right) ^{N-\alpha }%
}{\left( N-\alpha \right) }\left \Vert f^{\left( N\right) }\right \Vert
_{\infty }=\frac{\left( x-y\right) ^{N-\alpha }}{\Gamma \left( N-\alpha
+1\right) }\left \Vert f^{\left( N\right) }\right \Vert _{\infty }
\end{equation*}%
\begin{equation}
\leq \frac{\left( b-a\right) ^{N-\alpha }}{\Gamma \left( N-\alpha +1\right) }%
\left \Vert f^{\left( N\right) }\right \Vert _{\infty }.  \tag{134}
\label{134}
\end{equation}%
That is 
\begin{equation}
\left \Vert D_{x-}^{\alpha }f\right \Vert _{\infty ,\left[ a,x\right] }\leq 
\frac{\left( b-a\right) ^{N-\alpha }}{\Gamma \left( N-\alpha +1\right) }%
\left \Vert f^{\left( N\right) }\right \Vert _{\infty },  \tag{135}
\label{135}
\end{equation}%
and 
\begin{equation}
\underset{x\in \left[ a,b\right] }{\sup }\left \Vert D_{x-}^{\alpha }f\right
\Vert _{\infty ,\left[ a,x\right] }\leq \frac{\left( b-a\right) ^{N-\alpha }%
}{\Gamma \left( N-\alpha +1\right) }\left \Vert f^{\left( N\right) }\right
\Vert _{\infty }.  \tag{136}  \label{136}
\end{equation}%
Similarly we have 
\begin{equation}
\left( D_{\ast x}^{\alpha }f\right) \left( y\right) =\frac{1}{\Gamma \left(
N-\alpha \right) }\int_{x}^{y}\left( y-t\right) ^{N-\alpha -1}f^{\left(
N\right) }\left( t\right) dt,\text{ \ }\forall \text{ }y\in \left[ x,b\right]
.  \tag{137}  \label{137}
\end{equation}%
Thus we get%
\begin{equation}
\left \Vert \left( D_{\ast x}^{\alpha }f\right) \left( y\right) \right \Vert
\leq \frac{1}{\Gamma \left( N-\alpha \right) }\left( \int_{x}^{y}\left(
y-t\right) ^{N-\alpha -1}dt\right) \left \Vert f^{\left( N\right) }\right
\Vert _{\infty }\leq  \tag{138}  \label{138}
\end{equation}%
\begin{equation*}
\frac{1}{\Gamma \left( N-\alpha \right) }\frac{\left( y-x\right) ^{N-\alpha }%
}{\left( N-\alpha \right) }\left \Vert f^{\left( N\right) }\right \Vert
_{\infty }\leq \frac{\left( b-a\right) ^{N-\alpha }}{\Gamma \left( N-\alpha
+1\right) }\left \Vert f^{\left( N\right) }\right \Vert _{\infty }.
\end{equation*}%
Hence 
\begin{equation}
\left \Vert D_{\ast x}^{\alpha }f\right \Vert _{\infty ,\left[ x,b\right]
}\leq \frac{\left( b-a\right) ^{N-\alpha }}{\Gamma \left( N-\alpha +1\right) 
}\left \Vert f^{\left( N\right) }\right \Vert _{\infty },  \tag{139}
\label{139}
\end{equation}%
and 
\begin{equation}
\underset{x\in \left[ a,b\right] }{\sup }\left \Vert D_{\ast x}^{\alpha
}f\right \Vert _{\infty ,\left[ x,b\right] }\leq \frac{\left( b-a\right)
^{N-\alpha }}{\Gamma \left( N-\alpha +1\right) }\left \Vert f^{\left(
N\right) }\right \Vert _{\infty }.  \tag{140}  \label{140}
\end{equation}%
From (\ref{92}) and (\ref{93}) we get 
\begin{equation}
\underset{x\in \left[ a,b\right] }{\sup }\omega _{1}\left( D_{x-}^{\alpha }f,%
\frac{1}{n^{\beta }}\right) _{\left[ a,x\right] }\leq \frac{2\left \Vert
f^{\left( N\right) }\right \Vert _{\infty }}{\Gamma \left( N-\alpha
+1\right) }\left( b-a\right) ^{N-\alpha },  \tag{141}  \label{141}
\end{equation}%
and 
\begin{equation}
\underset{x\in \left[ a,b\right] }{\sup }\omega _{1}\left( D_{\ast
x}^{\alpha }f,\frac{1}{n^{\beta }}\right) _{\left[ x,b\right] }\leq \frac{%
2\left \Vert f^{\left( N\right) }\right \Vert _{\infty }}{\Gamma \left(
N-\alpha +1\right) }\left( b-a\right) ^{N-\alpha }.  \tag{142}  \label{142}
\end{equation}%
That is $E_{n}<\infty .$

We finally notice that 
\begin{equation*}
A_{n}\left( f,x\right) -\sum_{j=1}^{N-1}\frac{f^{\left( j\right) }\left(
x\right) }{j!}A_{n}\left( \left( \cdot -x\right) ^{j}\right) \left( x\right)
-f\left( x\right) =
\end{equation*}%
\begin{equation*}
\frac{A_{n}^{\ast }\left( f,x\right) }{\left( \sum_{k=\left \lceil na\right
\rceil }^{\left \lfloor nb\right \rfloor }\Phi \left( nx-k\right) \right) }-%
\frac{1}{\left( \sum_{k=\left \lceil na\right \rceil }^{\left \lfloor
nb\right \rfloor }\Phi \left( nx-k\right) \right) }\cdot
\end{equation*}%
\begin{equation*}
\left( \sum_{j=1}^{N-1}\frac{f^{\left( j\right) }\left( x\right) }{j!}%
A_{n}^{\ast }\left( \left( \cdot -x\right) ^{j}\right) \left( x\right)
\right) -f\left( x\right) =
\end{equation*}%
\begin{equation}
\frac{1}{\left( \sum_{k=\left \lceil na\right \rceil }^{\left \lfloor
nb\right \rfloor }\Phi \left( nx-k\right) \right) }\left[ A_{n}^{\ast
}\left( f,x\right) -\left( \sum_{j=1}^{N-1}\frac{f^{\left( j\right) }\left(
x\right) }{j!}A_{n}^{\ast }\left( \left( \cdot -x\right) ^{j}\right) \left(
x\right) \right) \right.  \tag{143}  \label{143}
\end{equation}%
\begin{equation*}
\left. -\left( \sum_{k=\left \lceil na\right \rceil }^{\left \lfloor
nb\right \rfloor }\Phi \left( nx-k\right) \right) f\left( x\right) \right] .
\end{equation*}%
Therefore we get 
\begin{equation*}
\left \Vert A_{n}\left( f,x\right) -\sum_{j=1}^{N-1}\frac{f^{\left( j\right)
}\left( x\right) }{j!}A_{n}\left( \left( \cdot -x\right) ^{j}\right) \left(
x\right) -f\left( x\right) \right \Vert \leq 2\left( \sqrt[2m]{1+4^{m}}%
\right) \cdot
\end{equation*}%
\begin{equation}
\left \Vert A_{n}^{\ast }\left( f,x\right) -\left( \sum_{j=1}^{N-1}\frac{%
f^{\left( j\right) }\left( x\right) }{j!}A_{n}^{\ast }\left( \left( \cdot
-x\right) ^{j}\right) \left( x\right) \right) -\left( \sum_{k=\left \lceil
na\right \rceil }^{\left \lfloor nb\right \rfloor }\Phi \left( nx-k\right)
\right) f\left( x\right) \right \Vert ,  \tag{144}  \label{144}
\end{equation}%
$\forall $ $x\in \left[ a,b\right] .$

The proof of the theorem is now completed.
\end{proof}

Next we apply Theorem \ref{t30.} for $N=1.$

\begin{theorem}
\label{t31.}Let $0<\alpha ,\beta <1$, $f\in C^{1}\left( \left[ a,b\right]
,X\right) $, $x\in \left[ a,b\right] $, $n\in \mathbb{N}:n^{1-\beta }>2$, $%
m\in \mathbb{N}.$ Then

i) 
\begin{equation*}
\left \Vert A_{n}\left( f,x\right) -f\left( x\right) \right \Vert \leq
\end{equation*}%
\begin{equation*}
\frac{2\left( \sqrt[2m]{1+4^{m}}\right) }{\Gamma \left( \alpha +1\right) }%
\left \{ \frac{\left( \omega _{1}\left( D_{x-}^{\alpha }f,\frac{1}{n^{\beta }%
}\right) _{\left[ a,x\right] }+\omega _{1}\left( D_{\ast x}^{\alpha }f,\frac{%
1}{n^{\beta }}\right) _{\left[ x,b\right] }\right) }{n^{\alpha \beta }}%
+\right.
\end{equation*}%
\begin{equation}
\left. \frac{1}{4m\left( n^{1-\beta }-2\right) ^{2m}}\left( \left \Vert
D_{x-}^{\alpha }f\right \Vert _{\infty ,\left[ a,x\right] }\left( x-a\right)
^{\alpha }+\left \Vert D_{\ast x}^{\alpha }f\right \Vert _{\infty ,\left[ x,b%
\right] }\left( b-x\right) ^{\alpha }\right) \right \} ,  \tag{145}
\label{145}
\end{equation}

and

ii) 
\begin{equation*}
\left \Vert A_{n}f-f\right \Vert _{\infty }\leq \frac{2\left( \sqrt[2m]{%
1+4^{m}}\right) }{\Gamma \left( \alpha +1\right) }\cdot
\end{equation*}%
\begin{equation*}
\left \{ \frac{\left( \underset{x\in \left[ a,b\right] }{\sup }\omega
_{1}\left( D_{x-}^{\alpha }f,\frac{1}{n^{\beta }}\right) _{\left[ a,x\right]
}+\underset{x\in \left[ a,b\right] }{\sup }\omega _{1}\left( D_{\ast
x}^{\alpha }f,\frac{1}{n^{\beta }}\right) _{\left[ x,b\right] }\right) }{%
n^{\alpha \beta }}+\right.
\end{equation*}%
\begin{equation}
\left. \frac{\left( b-a\right) ^{\alpha }}{4m\left( n^{1-\beta }-2\right)
^{2m}}\left( \underset{x\in \left[ a,b\right] }{\sup }\left \Vert
D_{x-}^{\alpha }f\right \Vert _{\infty ,\left[ a,x\right] }+\underset{x\in %
\left[ a,b\right] }{\sup }\left \Vert D_{\ast x}^{\alpha }f\right \Vert
_{\infty ,\left[ x,b\right] }\right) \right \} .  \tag{146}  \label{146}
\end{equation}
\end{theorem}

When $\alpha =\frac{1}{2}$ we derive

\begin{corollary}
\label{c32.}Let $0<\beta <1$, $f\in C^{1}\left( \left[ a,b\right] ,X\right) $%
, $x\in \left[ a,b\right] $, $n\in \mathbb{N}:n^{1-\beta }>2$, $m\in \mathbb{%
N}.$ Then

i) 
\begin{equation*}
\left \Vert A_{n}\left( f,x\right) -f\left( x\right) \right \Vert \leq
\end{equation*}%
\begin{equation*}
\frac{4\left( \sqrt[2m]{1+4^{m}}\right) }{\sqrt{\pi }}\left \{ \frac{\left(
\omega _{1}\left( D_{x-}^{\frac{1}{2}}f,\frac{1}{n^{\beta }}\right) _{\left[
a,x\right] }+\omega _{1}\left( D_{\ast x}^{\frac{1}{2}}f,\frac{1}{n^{\beta }}%
\right) _{\left[ x,b\right] }\right) }{n^{\frac{\beta }{2}}}+\right.
\end{equation*}%
\begin{equation}
\left. \frac{1}{4m\left( n^{1-\beta }-2\right) ^{2m}}\left( \left \Vert
D_{x-}^{\frac{1}{2}}f\right \Vert _{\infty ,\left[ a,x\right] }\sqrt{\left(
x-a\right) }+\left \Vert D_{\ast x}^{\frac{1}{2}}f\right \Vert _{\infty ,%
\left[ x,b\right] }\sqrt{\left( b-x\right) }\right) \right \} ,  \tag{147}
\label{147}
\end{equation}

and

ii) 
\begin{equation*}
\left \Vert A_{n}f-f\right \Vert _{\infty }\leq \frac{4\left( \sqrt[2m]{%
1+4^{m}}\right) }{\sqrt{\pi }}\cdot
\end{equation*}%
\begin{equation*}
\left \{ \frac{\left( \underset{x\in \left[ a,b\right] }{\sup }\omega
_{1}\left( D_{x-}^{\frac{1}{2}}f,\frac{1}{n^{\beta }}\right) _{\left[ a,x%
\right] }+\underset{x\in \left[ a,b\right] }{\sup }\omega _{1}\left( D_{\ast
x}^{\frac{1}{2}}f,\frac{1}{n^{\beta }}\right) _{\left[ x,b\right] }\right) }{%
n^{\frac{\beta }{2}}}+\right.
\end{equation*}%
\begin{equation}
\left. \frac{\sqrt{\left( b-a\right) }}{4m\left( n^{1-\beta }-2\right) ^{2m}}%
\left( \underset{x\in \left[ a,b\right] }{\sup }\left \Vert D_{x-}^{\frac{1}{%
2}}f\right \Vert _{\infty ,\left[ a,x\right] }+\underset{x\in \left[ a,b%
\right] }{\sup }\left \Vert D_{\ast x}^{\frac{1}{2}}f\right \Vert _{\infty ,%
\left[ x,b\right] }\right) \right \} <\infty .  \tag{148}  \label{148}
\end{equation}
\end{corollary}

Next we make

\begin{remark}
\label{r33}Some convergence analysis follows:

Let $0<\beta <1$, $f\in C^{1}\left( \left[ a,b\right] ,X\right) $, $x\in %
\left[ a,b\right] $, $n\in \mathbb{N}:n^{1-\beta }>2$, $m\in \mathbb{N}.$ We
elaborate on (\ref{148}). Assume that 
\begin{equation}
\omega _{1}\left( D_{x-}^{\frac{1}{2}}f,\frac{1}{n^{\beta }}\right) _{\left[
a,x\right] }\leq \frac{K_{1}}{n^{\beta }},  \tag{149}  \label{149}
\end{equation}%
and 
\begin{equation}
\omega _{1}\left( D_{\ast x}^{\frac{1}{2}}f,\frac{1}{n^{\beta }}\right) _{%
\left[ x,b\right] }\leq \frac{K_{2}}{n^{\beta }},  \tag{150}  \label{150}
\end{equation}%
$\forall $ $x\in \left[ a,b\right] $, $\forall $ $n\in \mathbb{N}$, where $%
K_{1},K_{2}>0$.

Then it holds 
\begin{equation*}
\frac{\left[ \underset{x\in \left[ a,b\right] }{\sup }\omega _{1}\left(
D_{x-}^{\frac{1}{2}}f,\frac{1}{n^{\beta }}\right) _{\left[ a,x\right] }+%
\underset{x\in \left[ a,b\right] }{\sup }\omega _{1}\left( D_{\ast x}^{\frac{%
1}{2}}f,\frac{1}{n^{\beta }}\right) _{\left[ x,b\right] }\right] }{n^{\frac{%
\beta }{2}}}\leq
\end{equation*}%
\begin{equation}
\frac{\frac{\left( K_{1}+K_{2}\right) }{n^{\beta }}}{n^{\frac{\beta }{2}}}=%
\frac{\left( K_{1}+K_{2}\right) }{n^{\frac{3\beta }{2}}}=\frac{K}{n^{\frac{%
3\beta }{2}}},  \tag{151}  \label{151}
\end{equation}%
where $K:=K_{1}+K_{2}>0.$

The other summand of the right hand side of (\ref{148}), for large enough $n$%
, converges to zero at the speed $\frac{1}{n^{2m\left( 1-\beta \right) }},$
so it is about $\frac{L}{n^{2m\left( 1-\beta \right) }}$, where $L>0$ is a
constant.

Then, for large enough $n\in \mathbb{N}$, by (\ref{148}), (\ref{151}) and
the above comment, we obtain that 
\begin{equation}
\left \Vert A_{n}f-f\right \Vert _{\infty }\leq \frac{M}{\min \left( n^{%
\frac{3\beta }{2}},n^{2m\left( 1-\beta \right) }\right) },  \tag{152}
\label{152}
\end{equation}%
where $M>0.$

Clearly we have two cases:

i) 
\begin{equation}
\left \Vert A_{n}f-f\right \Vert _{\infty }\leq \frac{M}{n^{2m\left( 1-\beta
\right) }}\text{, when }\frac{4m}{3+4m}\leq \beta <1,  \tag{153}  \label{153}
\end{equation}%
with speed of convergence $\frac{1}{n^{2m\left( 1-\beta \right) }},$

and

ii) 
\begin{equation}
\left \Vert A_{n}f-f\right \Vert _{\infty }\leq \frac{M}{n^{\frac{3\beta }{2}%
}},\text{ when }0<\beta \leq \frac{4m}{3+4m},  \tag{154}  \label{154}
\end{equation}%
with speed of convergence $\frac{1}{n^{\frac{3\beta }{2}}}.$

In Theorem \ref{t11.}, for $f\in C\left( \left[ a,b\right] ,X\right) $ and
for large enough $n\in \mathbb{N}$, when $0<\beta \leq \frac{2m}{1+2m}$, the
speed is $\frac{1}{n^{\beta }}$. So when $0<\beta \leq \frac{4m}{3+4m}$ ($<%
\frac{2m}{1+2m}$), we get by (\ref{154}) that $\left\Vert
A_{n}f-f\right\Vert _{\infty }$ converges much faster to zero. The last
comes because we assumed differentiability of $f$.

Notice that in Corollary \ref{c32.} no initial condition is assumed.
\end{remark}

Next, we will present an alternative fractional approximation by $A_{n}$, $%
n\in \mathbb{N}$.

\begin{notation}
\label{n34}Let $\overline{n}\in \mathbb{N}$, we denote the left iterated
fractional derivative 
\begin{equation}
D_{\ast x}^{\overline{n}x}=D_{\ast x}^{\alpha }D_{\ast x}^{\alpha
}...D_{\ast x}^{\alpha }\text{, \ (}\overline{n}\text{ - times),}  \tag{155}
\label{155}
\end{equation}%
$x\in \left[ a,b\right] $, $0<\alpha <1$.

Similarly, we also denote the right iterated fractional derivative 
\begin{equation}
D_{x-}^{\overline{n}x}=D_{x-}^{\alpha }D_{x-}^{\alpha }...D_{x-}^{\alpha }%
\text{, \ (}\overline{n}\text{ - times),}  \tag{156}  \label{156}
\end{equation}%
$x\in \left[ a,b\right] $.
\end{notation}

We need

\begin{definition}
\label{d35}Let $\overline{n}\in \mathbb{N}$, $D_{x}^{\left( \overline{n}%
+1\right) \alpha }f$ denote any of $D_{\ast x}^{\left( \overline{n}+1\right)
\alpha }$, $D_{x-}^{\left( \overline{n}+1\right) \alpha }$, and $\delta >0$.
We set 
\begin{equation}
\omega _{1}\left( D_{x}^{\left( \overline{n}+1\right) \alpha }f,\delta
\right) =\max \left \{ \omega _{1}\left( D_{\ast x}^{\left( \overline{n}%
+1\right) \alpha }f,\delta \right) _{\left[ x,b\right] },\omega _{1}\left(
D_{x-}^{\left( \overline{n}+1\right) \alpha }f,\delta \right) _{\left[ a,x%
\right] }\right \} ,  \tag{157}  \label{157}
\end{equation}%
where $x\in \left[ a,b\right] $. Here the moduli of continuity are
considered over $\left[ x,b\right] $ and $\left[ a,x\right] $, respectively.
\end{definition}

We also need

\begin{theorem}
\label{t36}(\cite{13}, p. 123) Let $0<\alpha <1$, $f:\left[ a,b\right]
\rightarrow \mathbb{R}$, $f^{\prime }\in L_{\infty }\left( \left[ a,b\right]
\right) $, $x\in \left[ a,b\right] $ being fixed. Assume that $D_{\ast
x}^{k\alpha }f\in C\left( \left[ x,b\right] \right) $, $k=0,1,...,\overline{n%
}+1$, $\overline{n}\in \mathbb{N}$, and $\left( D_{\ast x}^{i\alpha
}f\right) \left( x\right) =0$, $i=2,3,...,\overline{n}+1$. Also, suppose
that $D_{x-}^{k\alpha }f\in C\left( \left[ a,x\right] \right) $, for $%
k=0,1,...,\overline{n}+1$, and $\left( D_{x-}^{i\alpha }f\right) \left(
x\right) =0$, for $i=2,3,...,\overline{n}+1$. Then 
\begin{equation}
\left \vert f\left( \cdot \right) -f\left( x\right) \right \vert \leq \frac{%
\omega _{1}\left( D_{x}^{\left( \overline{n}+1\right) \alpha }f,\delta
\right) }{\Gamma \left( \left( \overline{n}+1\right) \alpha +1\right) }\left[
\left \vert \cdot -x\right \vert ^{\left( \overline{n}+1\right) \alpha }+%
\frac{\left \vert \cdot -x\right \vert ^{\left( \overline{n}+1\right) \alpha
+1}}{\delta \left( \left( \overline{n}+1\right) \alpha +1\right) }\right] ,%
\text{ \ }\delta >0.  \tag{158}  \label{158}
\end{equation}
\end{theorem}

We present

\begin{theorem}
\label{t37}Let $f\in C\left( \left[ a,b\right] \right) $ and all as in
Theorem \ref{t36}, $n\in \mathbb{N}:n^{1-\alpha }>2$, $m\in \mathbb{N}.$

Then 
\begin{equation*}
\left \vert \left( A_{n}f\right) \left( x\right) -f\left( x\right) \right
\vert \leq \frac{2\left( \sqrt[2m]{1+4^{m}}\right) \omega _{1}\left(
D_{x}^{\left( \overline{n}+1\right) \alpha }f,\delta \right) }{\Gamma \left(
\left( \overline{n}+1\right) \alpha +1\right) }
\end{equation*}%
\begin{equation}
\left \{ \left[ \frac{1}{n^{\left( \overline{n}+1\right) \alpha ^{2}}}+\frac{%
\left( b-a\right) ^{\left( \overline{n}+1\right) \alpha }}{4m\left(
n^{1-\alpha }-2\right) ^{2m}}\right] +\right.  \tag{159}  \label{159}
\end{equation}%
\begin{equation*}
\left. \frac{1}{\delta \left( \left( \overline{n}+1\right) \alpha +1\right) }%
\left[ \frac{1}{n^{\alpha \left[ \left( \overline{n}+1\right) \alpha +1%
\right] }}+\frac{\left( b-a\right) ^{\left( \overline{n}+1\right) \alpha +1}%
}{4m\left( n^{1-\alpha }-2\right) ^{2m}}\right] \right \} ,\text{ \ }\delta
>0.
\end{equation*}%
Hence $\underset{n\rightarrow +\infty }{\lim }A_{n}\left( f\right) \left(
x\right) =f\left( x\right) .$
\end{theorem}

\begin{proof}
We notice that $A_{n}$ is a positive linear operator with $A_{n}\left(
1\right) =1$.

Let $f\in C\left( \left[ a,b\right] ,\mathbb{R}\right) $, then $\left \vert
f\right \vert \leq \left \vert f\right \vert $ and $-\left \vert
f\right
\vert \leq f\leq \left \vert f\right \vert .$

Hence $-A_{n}\left( \left \vert f\right \vert \right) \leq A_{n}\left(
f\right) \leq A_{n}\left( \left \vert f\right \vert \right) $ and $%
\left
\vert A_{n}\left( f\right) \right \vert \leq A_{n}\left( \left \vert
f\right
\vert \right) $.

Therefore 
\begin{equation*}
\left\vert \left( A_{n}f\right) \left( x\right) -f\left( x\right)
\right\vert =\left\vert \left( A_{n}f\right) \left( x\right) -A_{n}\left(
f\left( x\right) \right) \left( x\right) \right\vert =
\end{equation*}%
\begin{equation}
\left\vert A_{n}\left( f-f\left( x\right) \right) \left( x\right)
\right\vert \overset{\text{(\ref{158})}}{\leq }A_{n}\left( \left\vert
f-f\left( x\right) \right\vert \right) \left( x\right) \leq   \tag{160}
\label{160}
\end{equation}%
\begin{equation*}
\frac{\omega _{1}\left( D_{x}^{\left( \overline{n}+1\right) \alpha }f,\delta
\right) }{\Gamma \left( \left( \overline{n}+1\right) \alpha +1\right) }\left[
A_{n}\left( \left\vert \cdot -x\right\vert ^{\left( \overline{n}+1\right)
\alpha }\right) \left( x\right) +\frac{A_{n}\left( \left\vert \cdot
-x\right\vert ^{\left( \overline{n}+1\right) \alpha +1}\right) \left(
x\right) }{\delta \left( \left( \overline{n}+1\right) \alpha +1\right) }%
\right] =
\end{equation*}%
\begin{equation*}
\frac{\omega _{1}\left( D_{x}^{\left( \overline{n}+1\right) \alpha }f,\delta
\right) }{\Gamma \left( \left( \overline{n}+1\right) \alpha +1\right)
\sum_{k=\left\lceil na\right\rceil }^{\left\lfloor nb\right\rfloor }\Phi
\left( nx-k\right) }
\end{equation*}%
\begin{equation}
\left[ \sum_{k=\left\lceil na\right\rceil }^{\left\lfloor nb\right\rfloor
}\left\vert \frac{k}{n}-x\right\vert ^{\left( \overline{n}+1\right) \alpha
}\Phi \left( nx-k\right) +\frac{\sum_{k=\left\lceil na\right\rceil
}^{\left\lfloor nb\right\rfloor }\left\vert \frac{k}{n}-x\right\vert
^{\left( \overline{n}+1\right) \alpha +1}\Phi \left( nx-k\right) }{\delta
\left( \left( \overline{n}+1\right) \alpha +1\right) }\right] \overset{\text{%
(\ref{26})}}{\leq }  \tag{161}  \label{161}
\end{equation}%
\begin{equation*}
\frac{2\left( \sqrt[2m]{1+4^{m}}\right) \omega _{1}\left( D_{x}^{\left( 
\overline{n}+1\right) \alpha }f,\delta \right) }{\Gamma \left( \left( 
\overline{n}+1\right) \alpha +1\right) }
\end{equation*}%
\begin{equation*}
\left\{ \left[ \sum_{\left\{ 
\begin{array}{l}
k=\left\lceil na\right\rceil  \\ 
:\left\vert \frac{k}{n}-x\right\vert \leq \frac{1}{n^{\alpha }}%
\end{array}%
\right. }^{\left\lfloor nb\right\rfloor }\left\vert \frac{k}{n}-x\right\vert
^{\left( \overline{n}+1\right) \alpha }\Phi \left( nx-k\right) +\right.
\right. 
\end{equation*}%
\begin{equation*}
\left. \sum_{\left\{ 
\begin{array}{l}
k=\left\lceil na\right\rceil  \\ 
:\left\vert \frac{k}{n}-x\right\vert >\frac{1}{n^{\alpha }}%
\end{array}%
\right. }^{\left\lfloor nb\right\rfloor }\left\vert \frac{k}{n}-x\right\vert
^{\left( \overline{n}+1\right) \alpha }\Phi \left( nx-k\right) \right] +
\end{equation*}%
\begin{equation}
\frac{1}{\delta \left( \left( \overline{n}+1\right) \alpha +1\right) }\left[
\sum_{\left\{ 
\begin{array}{l}
k=\left\lceil na\right\rceil  \\ 
:\left\vert \frac{k}{n}-x\right\vert \leq \frac{1}{n^{\alpha }}%
\end{array}%
\right. }^{\left\lfloor nb\right\rfloor }\left\vert \frac{k}{n}-x\right\vert
^{\left( \overline{n}+1\right) \alpha +1}\Phi \left( nx-k\right) +\right.  
\tag{162}  \label{162}
\end{equation}%
\begin{equation*}
\left. \left. \sum_{\left\{ 
\begin{array}{l}
k=\left\lceil na\right\rceil  \\ 
:\left\vert \frac{k}{n}-x\right\vert >\frac{1}{n^{\alpha }}%
\end{array}%
\right. }^{\left\lfloor nb\right\rfloor }\left\vert \frac{k}{n}-x\right\vert
^{\left( \overline{n}+1\right) \alpha +1}\Phi \left( nx-k\right) \right]
\right\} \overset{\text{(\ref{19})}}{\leq }
\end{equation*}%
\begin{equation*}
\frac{2\left( \sqrt[2m]{1+4^{m}}\right) \omega _{1}\left( D_{x}^{\left( 
\overline{n}+1\right) \alpha }f,\delta \right) }{\Gamma \left( \left( 
\overline{n}+1\right) \alpha +1\right) }\left\{ \left[ \frac{1}{n^{\left( 
\overline{n}+1\right) \alpha ^{2}}}+\frac{\left( b-a\right) ^{\left( 
\overline{n}+1\right) \alpha }}{4m\left( n^{1-\alpha }-2\right) ^{2m}}\right]
\right. 
\end{equation*}%
\begin{equation}
\left. +\frac{1}{\delta \left( \left( \overline{n}+1\right) \alpha +1\right) 
}\left[ \frac{1}{n^{\alpha \left[ \left( \overline{n}+1\right) \alpha +1%
\right] }}+\frac{\left( b-a\right) ^{\left( \overline{n}+1\right) \alpha +1}%
}{4m\left( n^{1-\alpha }-2\right) ^{2m}}\right] \right\} ,\ \delta >0, 
\tag{163}  \label{163}
\end{equation}%
proving the claim.
\end{proof}

We finish with

\begin{corollary}
\label{c38}All as in Theorem \ref{t37}, with $\delta =\frac{1}{\left( 
\overline{n}+1\right) \alpha +1}.$ Then 
\begin{equation*}
\left\vert \left( A_{n}f\right) \left( x\right) -f\left( x\right)
\right\vert \leq \frac{2\left( \sqrt[2m]{1+4^{m}}\right) \omega _{1}\left(
D_{x}^{\left( \overline{n}+1\right) \alpha }f,\frac{1}{\left( \overline{n}%
+1\right) \alpha +1}\right) }{\Gamma \left( \left( \overline{n}+1\right)
\alpha +1\right) }
\end{equation*}%
\begin{equation}
\left\{ \left[ \frac{1}{n^{\left( \overline{n}+1\right) \alpha ^{2}}}+\frac{%
\left( b-a\right) ^{\left( \overline{n}+1\right) \alpha }}{4m\left(
n^{1-\alpha }-2\right) ^{2m}}\right] +\left[ \frac{1}{n^{\alpha \left[
\left( \overline{n}+1\right) \alpha +1\right] }}+\frac{\left( b-a\right)
^{\left( \overline{n}+1\right) \alpha +1}}{4m\left( n^{1-\alpha }-2\right)
^{2m}}\right] \right\} .  \tag{164}  \label{164}
\end{equation}%
Hence $\underset{n\rightarrow +\infty }{\lim }A_{n}\left( f\right) \left(
x\right) =f\left( x\right) .$
\end{corollary}

\end{document}